\documentclass{article}

\usepackage[preprint]{neurips_2023}

\usepackage{natbib}
\setcitestyle{numbers,square}

\usepackage[utf8]{inputenc} 
\usepackage[T1]{fontenc}    
\usepackage{hyperref}       
\usepackage{url}            
\usepackage{booktabs}       
\usepackage{amsfonts}       
\usepackage{nicefrac}       
\usepackage{microtype}      
\usepackage{xcolor}         

\usepackage{times}
\usepackage{epsfig}
\usepackage{graphicx}
\usepackage{amsmath}
\usepackage{amssymb}

\usepackage{amsmath}
\usepackage{amssymb}
\usepackage{mathtools}
\usepackage{amsthm}

\usepackage{multirow}
\usepackage{algorithm}
\usepackage{algorithmic}

\newtheorem{lemma}{Lemma}
\newtheorem{remark}{Remark}
\newtheorem{proposition}{Proposition}%
\newtheorem{definition}{Definition}%

\usepackage{xcolor}
\usepackage{xspace} 
\definecolor{purple}{rgb}{0.2, 0, 1.0}
\definecolor{darkred}{rgb}{0.8, 0.1, 0.1}
{}  
{}

\title{Exact Diffusion Inversion via Bi-directional Integration Approximation}

\author{%
  Guoqiang Zhang  \\
  University of Exeter \\
  \texttt{g.z.zhang@exeter.ac.uk} \\
   \And  J.P. Lewis
    \\
   Nvidia \\
\texttt{jpl@nvidia.com} \\
   \And
   W. Bastiaan Kleijn \\ 
Victoria University of Wellington \\ \texttt{bastiaan.kleijn@vuw.ac.nz} \\
}

\begin{document}

\maketitle
\begin{abstract}

Recently, various methods have been proposed to address the inconsistency issue of DDIM inversion to enable image editing, such as EDICT \cite{Wallace23EDICT} and Null-text inversion \cite{Mokady23NullTestInv}. However, the above methods introduce considerable computational overhead. In this paper, we propose a new  
technique, named \emph{bi-directional integration approximation} (BDIA), to perform exact diffusion inversion with neglible computational overhead. Suppose we would like to estimate the next diffusion state $\boldsymbol{z}_{i-1}$ at timestep $t_i$ with the historical information $(i,\boldsymbol{z}_i)$ and $(i+1,\boldsymbol{z}_{i+1})$. We first obtain the estimated Gaussian noise $\hat{\boldsymbol{\epsilon}}(\boldsymbol{z}_i,i)$, and then apply the DDIM update procedure twice for approximating the ODE integration over the next time-slot $[t_i, t_{i-1}]$ in the forward manner and the previous time-slot $[t_i, t_{t+1}]$ in the backward manner. The DDIM step for the previous time-slot is used to refine the integration approximation made earlier when computing $\boldsymbol{z}_i$. A nice property of BDIA-DDIM is that the update expression for $\boldsymbol{z}_{i-1}$ is a linear combination of $(\boldsymbol{z}_{i+1}, \boldsymbol{z}_i, \hat{\boldsymbol{\epsilon}}(\boldsymbol{z}_i,i))$. This allows for exact backward computation of $\boldsymbol{z}_{i+1}$ given $(\boldsymbol{z}_i, \boldsymbol{z}_{i-1})$, thus leading to exact diffusion inversion. It is demonstrated with experiments that (round-trip) BDIA-DDIM is particularly effective for image editing.  
Our experiments further show that BDIA-DDIM produces markedly better image sampling qualities than DDIM for text-to-image generation, thanks to the more accurate integration approximation. 

BDIA can also be applied to improve the  performance of other ODE solvers in addition to DDIM. In our work, it is found that applying BDIA to the EDM sampling procedure produces consistently better performance over four pre-trained models.


\end{abstract}

\section{Introduction}

As one type of generative models, diffusion probabilistic models (DPMs) have made significant progress in recent years. The pioneering work \cite{Dickstein15DPM} applied non-equilibrium statistical physics to the estimation of probabilistic data distributions. In doing so, a Markov forward diffusion process is constructed by systematically inserting additive noise into a data sample until the data distribution is almost destroyed. The data distribution is then gradually restored by a reverse diffusion process starting from a simple parametric distribution. The main advantage of DPM over classic tractable models (e.g., HMMs, GMMs, see \cite{Bishop06}) is that DPM can accurately model both the high and low likelihood regions of the data distribution by estimating a sequence of progressively less noise-perturbed data distributions. In comparison to generative adversarial networks (GANs) \cite{Goodfellow14GAN, Arjovsky17WGAN, Gulrajani17WGANGP, Sauer22StyleGAN}, DPMs exhibit more stable training dynamics by avoiding adversarial learning, as well as showing better sample diversity.

\begin{figure}[t!]
\centering
\includegraphics[width=120mm]{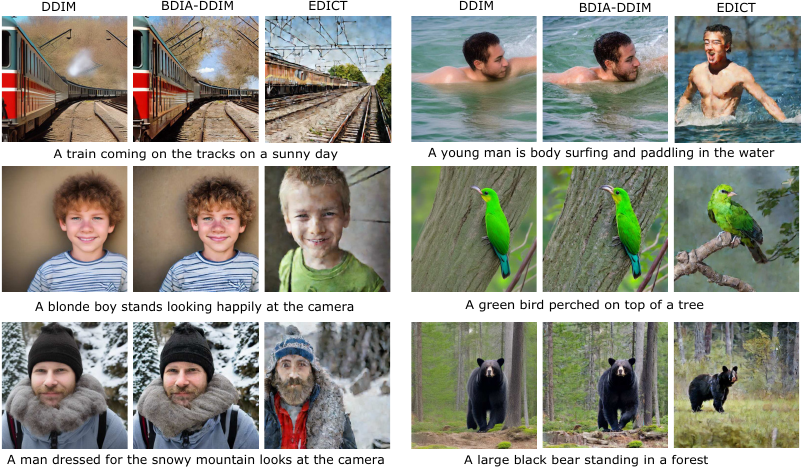}
\vspace*{-0.2cm}
\caption{ Comparison of generated images by using BDIA-DDIM, DDIM, and EDICT with 10 timesteps over StableDiffusion V2. See Subsection \ref{subsec:exp_imgSampling} for details. } 
\label{fig:BDIADDIM_t2i}
\vspace*{-0.0cm}
\end{figure}

Following the work of \cite{Dickstein15DPM}, various learning and/or sampling strategies have been proposed to improve the performance of DPMs, which include, for example, denoising diffusion probabilistic models (DDPMs) \cite{Ho20DDPM}, denoising diffusion implicit models (DDIMs) \cite{Song21DDIM}, improved DDIMs \cite{Nichol21DDPM, Dhariwal21DPM}, latent diffusion models (LDMs)\cite{Rombach22LDM}, score matching with Langevin dynamics (SMLD) \cite{Song19, Song21DPM, Song21SDE_gen}, analytic-DPMs \cite{Bao22DPM, Bao22DPM_cov},  optimized denoising schedules \cite{Kingma21DDPM, Chen20WaveGrad, Lam22BDDM}, guided diffusion strategies \cite{Nichol22GLIDE, Kim22GuidedDiffusion}, and classifier-free guided diffusion \cite{Ho22ClassiferFreeGuide}. It is worth noting that DDIM can be interpreted as a first-order ODE solver. As an extension of DDIM, various high-order ODE solvers have been proposed, such as EDM \cite{Karras22EDM}, DEIS \cite{Zhang22DEIS}, PNDM \cite{Liu22PNDM}, DPM-Solvers \cite{Lu22DPM_Solver}, and IIA-EDM and IIA-DDIM \cite{GuoqiangIIA23}.

In recent years, image-editing via diffusion models has attracted increasing attention in both academia and industry. One important operation for editing a real image is to first perform forward process on the image to obtain the final noise representation and then perform a backward process with embedded editing to generate the desired image \cite{Rombach22StableDiffusion, Saharia22Imagen}. DDIM inversion has been widely used to perform the above forward and backward processes \cite{Shi23DragDiffusion}. A major issue with DDIM inversion is that the intermediate diffusion states in the forward and backward processes may be inconsistent due to the inherent approximations (see Subsection~\ref{subsec:DDIM_review}). This issue becomes significant when utilizing classifier-free guidance in text-to-image editing \cite{Shi23DragDiffusion}.  The newly generated images are often perceptually far away from the original ones, which is undesirable for image-editing.

Recently, two methods have been proposed to address the inconsistency issue of DDIM inversion. Specifically, 
the work of \cite{Mokady23NullTestInv} proposed a technique named \emph{null-text inversion} to push the diffusion states of the backward process to be optimally close to those of the forward process via iterative optimization. The null-text inputs to the score neural network are treated as free variables in the optimization procedure.
In \cite{Wallace23EDICT}, the authors proposed the EDICT technique to enforce exact DDIM inversion. Their basic idea is to introduce an auxiliary diffusion state and then perform alternating updates on the primal and auxiliary diffusion states, which is inspired by the flow generative framework \cite{Kingma18Glow,Dinh14Nice,Dinh16DensityEsti}. One drawback of EDICT is that the number of neural function
evaluations (NFEs) has to be doubled in comparison to DDIM inversion (See Subsection~\ref{subsec:EDICT}). Another related line of research work is DDPM inversion (see \cite{Huberman23DDPMInversion}).

In this paper, we propose a new  technique to enforce exact DDIM inversion with negligible computational overhead, reducing the number of NFEs required in EDICT by half. Suppose we are in a position to estimate the next diffusion state $\boldsymbol{z}_{i-1}$ at timestep $t_i$ by utilizing the two most recent states $\boldsymbol{z}_i$ and  $\boldsymbol{z}_{i+1}$. With the estimated Gaussian noise $\hat{\boldsymbol{\epsilon}}(\boldsymbol{z}_i,i)$, we perform the DDIM update procedure twice for approximating the ODE integration over the next time-slot $[t_i, t_{i-1}]$ in the forward manner and the previous time-slot $[t_i,t_{i+1}]$ in the backward manner. The DDIM for the previous time-slot is employed to refine the integration approximation made earlier when computing $\boldsymbol{z}_i$. As a result, the expression for $\boldsymbol{z}_{i-1}$ becomes a linear combination of $(\boldsymbol{z}_{i+1}, \boldsymbol{z}_{i},\hat{\boldsymbol{\epsilon}}(\boldsymbol{z}_i,i))$, and naturally facilitates exact diffusion inversion.  We refer to the above  technique as \emph{bi-directional integration approximation (BDIA)}. 
Experiments demonstrate that BDIA-DDIM produces promising results on both image sampling and image editing (Figs.~\ref{fig:BDIADDIM_t2i},~\ref{fig:IIA},~\ref{fig:BDIA_DDIM}, Table~\ref{tab:BDIA_DDIM}). We have also applied BDIA to EDM, and found that the image qualities are also improved considerably (Table~\ref{tab:BDIA_EDM}). 



\section{Preliminary}
\textbf{Forward and reverse diffusion processes:}
Suppose the data sample $\boldsymbol{x}\in \mathbb{R}^d$ follows a data distribution $p_{data}(\boldsymbol{x})$ with a bounded variance. A forward diffusion process progressively adds Gaussian noise to the data samples $\boldsymbol{x}$ to obtain $\boldsymbol{z}_t$ as $t$ increases from 0 until $T$. The conditional distribution  of $\boldsymbol{z}_t$ given $\boldsymbol{x}$ can be represented as
\begin{align}
q_{t|0}(\boldsymbol{z}_t|\boldsymbol{x}) = \mathcal{N}(\boldsymbol{z}_t|\alpha_t\boldsymbol{x}, \sigma_t^2\boldsymbol{I})\quad \boldsymbol{z}_t = \alpha_t\boldsymbol{x}+\sigma_t \boldsymbol{\epsilon},
\label{equ:forwardGaussian}
\end{align}
where $\alpha_t$ and $\sigma_t$ are assumed to be differentiable functions of $t$ with bounded derivatives. We use $q(\boldsymbol{z}_t; \alpha_t,\sigma_t)$ to denote the marginal distribution of $\boldsymbol{z}_t$. The samples of the distribution $q(\boldsymbol{z}_T;\alpha_T,\sigma_T)$ should be practically indistinguishable from pure Gaussian noise if $\sigma_T \gg \alpha_T$.

The reverse process of a diffusion model firstly draws a sample $\boldsymbol{z}_T$ from $\mathcal{N}(\boldsymbol{0}, {\sigma}_{T}^2\boldsymbol{I})$, and then progressively denoises it to obtain a sequence of diffusion states  $\{\boldsymbol{z}_{t_i}\sim p(\boldsymbol{z};\alpha_{t_i},\sigma_{t_i})\}_{i=0}^N$,
where we use the notation $p(\cdot)$ to indicate that reverse sample distribution might not be identical to the forward distribution $q(\cdot)$ because of practical approximations.  It is expected that the final sample $\boldsymbol{z}_{t_0}$ is roughly distributed according to $p_{data}(\boldsymbol{x})$, i.e., $p_{data}(\boldsymbol{x})\approx p(\boldsymbol{z}_{t_0};\alpha_{t_0},\sigma_{t_0})$ where $t_0=0$.

\textbf{ODE formulation:} In \cite{Song21SDE_gen},  Song et al. present a so-called \emph{probability flow} ODE which shares the same marginal distributions as $\boldsymbol{z}_t$ in (\ref{equ:forwardGaussian}).  Specifically, with the formulation (\ref{equ:forwardGaussian}) for a forward diffusion process, its reverse ODE form can be represented as 
\begin{align}
d\boldsymbol{z} =& \underbrace{\left[f(t)\boldsymbol{z}_t-\frac{1}{2}g^2(t)\nabla_{\boldsymbol{z}}\log q(\boldsymbol{z}_t; \alpha_t,\sigma_t)\right]}_{\boldsymbol{d}(\boldsymbol{z}_t, t)}dt,
\label{equ:ODE_gen}
\end{align}
where $\boldsymbol{d}(\boldsymbol{z}_t,t)$ denotes the gradient vector at time $t$, and the two functions $f(t)$ and $g(t)$ are represented in terms of $(\alpha_t, \sigma_t)$ as
\begin{align}
f(t) = \frac{d\log \alpha_t}{dt},\quad g^2(t)=\frac{d\sigma_t^2}{dt}-2\frac{d\log\alpha_t}{dt}\sigma_t^2. \label{equ:ODE_gen2}
\end{align}
 $\nabla_{\boldsymbol{z}}\log q(\boldsymbol{z};\alpha_t,\sigma_t)$ in (\ref{equ:ODE_gen}) is the score function  \cite{Hyvarinen05ScoreMatching} pointing towards higher density of data samples at the given noise level $(\alpha_t,\sigma_t)$. One nice property of the score function is that it does not depend on the generally intractable normalization constant of the underlying density function $q(\boldsymbol{z};\alpha_t,\sigma_t)$. 

As $t$ increases, the probability flow ODE (\ref{equ:ODE_gen}) continuously reduces the noise level of the data samples in the reverse process. In the ideal scenario where no approximations are introduced in (\ref{equ:ODE_gen}), the sample distribution $p(\boldsymbol{z};\alpha_t,\sigma_t)$ approaches $p_{data}(\boldsymbol{x})$ as $t$ goes from $T$ to 0. As a result, the sampling process of a diffusion model boils down to solving the ODE form (\ref{equ:ODE_gen}), where randomness is only introduced in the initial sample at time $T$. This has opened up the research opportunity of exploiting different ODE solvers in diffusion-based sampling processes.     

\textbf{Denoising score matching:} To be able to utilize (\ref{equ:ODE_gen}) for sampling, one needs to specify a particular form of the score function $\nabla_{\boldsymbol{z}}\log q(\boldsymbol{z};\alpha_t,\sigma_t)$.  One common approach is to train a noise estimator $ \hat{\boldsymbol{\epsilon}}_{\boldsymbol{\theta}}$ by minimizing the expected $L_2$ error for samples drawn from $q_{data}$ (see \cite{Ho20DDPM, Song21SDE_gen, Song21DDIM}):
\begin{align}
\mathbb{E}_{\boldsymbol{x}\sim p_{data}}\mathbb{E}_{\boldsymbol{\epsilon}\sim \mathcal{N}(\boldsymbol{0}, \sigma_t^2\boldsymbol{I})}\|\hat{\boldsymbol{\epsilon}}_{\boldsymbol{\theta}}(\alpha_t \boldsymbol{x}+\sigma_t\boldsymbol{\epsilon},t)-\boldsymbol{\epsilon}\|_2^2,\label{equ:epsilon_training}
\end{align}
where $(\alpha_t, \sigma_t)$ are from the forward process (\ref{equ:forwardGaussian}). The common practice in diffusion models is to utilize a neural network of U-Net architecture \cite{Ronneberger15Unet} to represent the noise estimator $\hat{\boldsymbol{\epsilon}}_{\boldsymbol{\theta}}$. With (\ref{equ:epsilon_training}), the score function can then be represented in terms of 
$\hat{\boldsymbol{\epsilon}}_{\boldsymbol{\theta}}(\boldsymbol{z}_t; t)$ as (see also (229) of \cite{Karras22EDM})
\begin{align}
\nabla_{\boldsymbol{z}}\log q(\boldsymbol{z}_t;\alpha_t,\sigma_t) =\frac{-(\boldsymbol{z}_t-\alpha_t \boldsymbol{x})}{\sigma_t^2}  = -\hat{\boldsymbol{\epsilon}}_{\boldsymbol{\theta}}(\boldsymbol{z}_t; t)/\sigma_t. \label{equ:}
\end{align}

Alternatively, the score function can be represented in terms of an estimator for $\boldsymbol{x}$ (see \cite{Karras22EDM}). The functional form for the noise level $(\alpha_t,\sigma_t)$ also plays an important role in the sampling quality in practice. For example, the setup $(\alpha_t,\sigma_t)=(1,\sqrt{t})$ was studied in \cite{Song21SDE_gen}, which corresponds to constant-speed heat diffusion. The recent work \cite{Karras22EDM} found that a simple form of $(\alpha_t,\sigma_t)=(1,t)$ works well in practice.

\begin{figure}[t!]
\centering
\includegraphics[width=120mm]{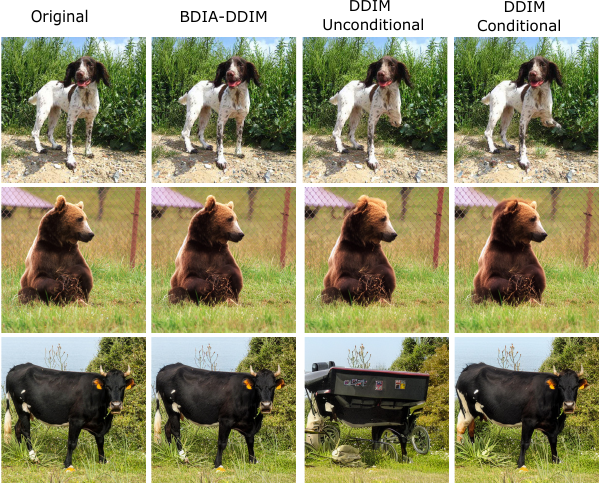}
\vspace*{-0.2cm}
\caption{ BDIA-DDIM produces perceptually better image reconstruction than DDIM. 
} 
\label{fig:IIA}
\vspace*{-0.0cm}
\end{figure}

\section{Bi-directional Integration Approximation (BDIA) for DDIM}  
In this section, we first review DDIM inversion and EDICT as an extension of DDIM inversion. We then
present our BDIA technique to enable exact diffusion inversion.  

\subsection{Review of DDIM inversion} 
\label{subsec:DDIM_review}

We first consider the update expression of DDIM for sampling, which is in fact a first-order solver for the ODE formulation (\ref{equ:ODE_gen})-(\ref{equ:ODE_gen2}) (see \cite{Lu22DPM_Solver,Zhang22DEIS}), given by 
\begin{align}
    \boldsymbol{z}_{i-1} \hspace{-0.2mm}=& \hspace{0.7mm} \alpha_{i-1} \left(\frac{\boldsymbol{z}_i \hspace{-0.3mm}-\hspace{-0.3mm} \sigma_{i}\hat{\boldsymbol{\epsilon}}_{\boldsymbol{\theta}}(\boldsymbol{z}_i, i) }{\alpha_{i}}\right)+\hspace{0.5mm}\sigma_{i-1}\hat{\boldsymbol{\epsilon}}_{\boldsymbol{\theta}}(\boldsymbol{z}_i, i) \label{equ:DDIM1}
    \\
    =& a_i \boldsymbol{z}_i +b_i\hat{\boldsymbol{\epsilon}}_{\boldsymbol{\theta}}(\boldsymbol{z}_i, i)  \label{equ:DDIM2} \\
    \approx& \boldsymbol{z}_{i}+\int_{t_i}^{t_{i-1}}\boldsymbol{d}(\boldsymbol{z}_{\tau},\tau)d\tau
    \label{equ:DDIM3},
\end{align}
where $a_i=\alpha_{i-1}/\alpha_i$ and $b_i=\sigma_{i-1}-\sigma_i\alpha_{i-1}/\alpha_i$. It is clear from (\ref{equ:DDIM1})-(\ref{equ:DDIM3}) that the integration $\int_{t_i}^{t_{i-1}}\boldsymbol{d}(\boldsymbol{z}_{\tau},\tau)d\tau$ is approximated by the forward DDIM update. That is, only the diffusion state $\boldsymbol{z}_i$ at the starting timestep $t_i$ is used in the integration approximation.

To perform DDIM inversion, $\boldsymbol{z}_i$ can be approximated in terms of $\boldsymbol{z}_{i-1}$ as
\begin{align}
\boldsymbol{z}_i &=\alpha_i\left(\frac{\boldsymbol{z}_{i-1}-\sigma_{i-1}\hat{\boldsymbol{\epsilon}}_{\boldsymbol{\theta}}(\boldsymbol{z}_i,i)}{\alpha_{i-1}}\right)+\sigma_i\hat{\boldsymbol{\epsilon}}_{\boldsymbol{\theta}}(\boldsymbol{z}_i,i) \label{equ:DDIM_inv1} \\
&\approx\alpha_i\left(\frac{\boldsymbol{z}_{i-1}-\sigma_{i-1}\hat{\boldsymbol{\epsilon}}_{\boldsymbol{\theta}}(\boldsymbol{z}_{i-1},i)}{\alpha_{i-1}}\right)+\sigma_i\hat{\boldsymbol{\epsilon}}_{\boldsymbol{\theta}}(\boldsymbol{z}_{i-1},i), \label{equ:DDIM_inv2}
\end{align}
where $\boldsymbol{z}_i$ in the RHS of (\ref{equ:DDIM_inv1}) is replaced with $\boldsymbol{z}_{i-1}$ to facilitate explicit computation. This naturally introduces approximation errors, leading to inconsistency of the diffusion states between the forward and backward processes. 

\subsection{Review of EDICT for exact diffusion inversion}
\label{subsec:EDICT}
Inspired by the flow generative framework \cite{Kingma18Glow}, the recent work \cite{Wallace23EDICT} proposed EDICT to enforce exact diffusion inversion. The basic idea is to introduce an auxiliary diffusion state $\boldsymbol{y}_i$ to be coupled with  $\boldsymbol{z}_i$ at every timestep $i$. The next pair of diffusion states $(\boldsymbol{z}_{i-1}, \boldsymbol{y}_{i-1})$ is then computed in an alternating fashion as  
\begin{align}
\boldsymbol{z}_{i}^{\text{inter}} &= a_i\boldsymbol{z}_{i} + b_i\boldsymbol{\epsilon}_{\boldsymbol{\theta}}(\boldsymbol{y}_{i},i) \label{equ:EDICT_r1} \\
\boldsymbol{y}_{i}^{\text{inter}} &= a_i\boldsymbol{y}_{i} + b_i\boldsymbol{\epsilon}_{\boldsymbol{\theta}}(\boldsymbol{z}_{i}^{\text{inter}},i) \label{equ:EDICT_r2}  \\
\boldsymbol{z}_{i-1} &= p\boldsymbol{z}_{i}^{\text{inter}}+(1-p)\boldsymbol{y}_{i}^{\text{inter}} \label{equ:EDICT_r3}  \\
\boldsymbol{y}_{i-1} &= p\boldsymbol{y}_{i}^{\text{inter}}+(1-p)\boldsymbol{z}_{i-1}, \label{equ:EDICT_r4}  
\end{align}
where $p\in [0,1]$ is the weighting factor in the mixing operations and the pair $(\boldsymbol{z}_i^{\text{inter}}, \boldsymbol{y}_i^{\text{inter}})$ represents the intermediate diffusion states. According to \cite{Wallace23EDICT}, the two mixing operations (\ref{equ:EDICT_r3})-(\ref{equ:EDICT_r4}) are introduced to make the update procedure stable. 

Due to the alternating update formalism in (\ref{equ:EDICT_r1})-(\ref{equ:EDICT_r4}), 
the computation can be inverted to obtain $(\boldsymbol{z}_i, \boldsymbol{y}_i)$ in terms of $(\boldsymbol{z}_{i-1}, \boldsymbol{y}_{i-1})$ as

\begin{align}
\boldsymbol{y}_{i}^{\text{inter}} &= (\boldsymbol{y}_{i-1}-(1-p)\boldsymbol{z}_{i-1})/p \label{equ:EDICT_f1}  \\
\boldsymbol{z}_{i}^{\text{inter}} &= (\boldsymbol{z}_{i-1}-(1-p)\boldsymbol{y}_{i}^{\text{inter}})/p \label{equ:EDICT_f2}  \\
\boldsymbol{y}_{i} &= (\boldsymbol{y}_{i}^{\text{inter}} - b_i\boldsymbol{\epsilon}_{\boldsymbol{\theta}}(\boldsymbol{z}_{i}^{\text{inter}},i))/a_i \label{equ:EDICT_f3}  \\
\boldsymbol{x}_{i} &= (\boldsymbol{z}_{i}^{\text{inter}} - b_i\boldsymbol{\epsilon}_{\boldsymbol{\theta}}(\boldsymbol{y}_{i},i)/a_i \label{equ:EDICT_f4} 
\end{align}
Unlike (\ref{equ:DDIM_inv1})-(\ref{equ:DDIM_inv2}), the inversion of (\ref{equ:EDICT_r1})-(\ref{equ:EDICT_r4}) does not involve any approximation, thus enabling exact diffusion inversion.  

However, it is clear from the above equations that the NFE that EDICT has to perform is two times the NFE required for DDIM. This makes the method computationally expensive in practice. It is highly desirable to reduce the NFE in EDICT while retaining exact diffusion inversion. We provide such a method in the next subsection.

\subsection{BDIA-DDIM for exact diffusion inversion}
\noindent \textbf{Reformulation of DDIM update expression}: In this section, we present our new technique BDIA to assist DDIM in achieving exact diffusion inversion. To do so, we first reformulate the update expression for $\boldsymbol{z}_{i-1}$ in (\ref{equ:DDIM2}) in terms of all the historical diffusion states $\{\boldsymbol{z}_{j}\}_{j=N}^i$ as
\begin{align}
\boldsymbol{z}_{i-1} &=\boldsymbol{z}_N+\sum_{j=N}^{i} \Delta(t_j\rightarrow t_{j-1}|\boldsymbol{z}_j) \label{equ:ddim_f_all0} \\
&\approx\boldsymbol{z}_N+\sum_{j=N}^i\int_{t_j}^{t_{j-1}} \boldsymbol{d}(\boldsymbol{z}_{\tau}, \tau)d\tau ,
\label{equ:ddim_f_all1}
\end{align}
where we use $\Delta(t_j\rightarrow t_{j-1}|\boldsymbol{z}_j) $ to denote approximation of the integration $\int_{t_j}^{t_{j-1}} \boldsymbol{d}(\boldsymbol{z}_{\tau},\tau)d\tau$ via the forward DDIM step, given by
\begin{align}
\Delta(t_j\rightarrow t_{j-1}|\boldsymbol{z}_j) &=\boldsymbol{z}_{j-1} - \boldsymbol{z}_j \nonumber \\
&=a_j\boldsymbol{z}_{j} + b_j\hat{\boldsymbol{\epsilon}}_{\boldsymbol{\theta}}(\boldsymbol{z}_{j},j)-\boldsymbol{z}_j. \label{equ:ddim_f}
\end{align}

\noindent\textbf{Extending forward DDIM by average of forward and backward DDIM}: We argue that, in principle,  each integration  $\int_{t_j}^{t_{j-1}} \boldsymbol{d}(\boldsymbol{z}_{\tau},\tau)d\tau$ in (\ref{equ:ddim_f_all1}) can be alternatively approximated by taking average of both the forward and backward DDIM updates, expressed as   
\begin{align}
\int_{t_{j}}^{t_{j-1}}\boldsymbol{d}(\boldsymbol{z}_{\tau},\tau)d\tau \approx (1-\phi_{i,j}) \Delta(t_{j}\rightarrow  t_{j-1}|\boldsymbol{z}_{j})  - \phi_{i,j} \Delta(t_{j-1}\rightarrow t_{j}|\boldsymbol{z}_{j-1}), 
\label{equ:ddim_fb}
\end{align}
where the scalar $\phi_{i,j}\in [0,1]$, and the notation $\Delta(t_{j-1}\rightarrow t_{j}|\boldsymbol{z}_{j-1})$ denotes the backward DDIM step from $t_{j-1}$ to $t_j$. The minus sign in front of $\Delta(t_{j-1}\rightarrow t_{j}|\boldsymbol{z}_{j-1})$ is due to integration over reverse time. The update expression for the backward DDIM step can be represented as 
\begin{align}
\Delta(t_{j-1}\rightarrow t_{j}|\boldsymbol{z}_{j-1}) 
&=\alpha_{j} \left(\frac{\boldsymbol{z}_{j-1} \hspace{-0.3mm}-\hspace{-0.3mm} \sigma_{j-1}\hat{\boldsymbol{\epsilon}}_{\boldsymbol{\theta}}(\boldsymbol{z}_{j-1}, j-1) }{\alpha_{j-1}}\right)+\hspace{0.5mm}\sigma_{j}\hat{\boldsymbol{\epsilon}}_{\boldsymbol{\theta}}(\boldsymbol{z}_{j-1}, j-1) -\boldsymbol{z}_{j-1} \label{equ:ddim_fb1} \\
&=\frac{\boldsymbol{z}_{j-1}}{a_j} - \frac{b_j}{a_j}\hat{\boldsymbol{\epsilon}}_{\boldsymbol{\theta}}(\boldsymbol{z}_{j-1},j-1) - \boldsymbol{z}_{j-1}. \label{equ:ddim_b}
\end{align}
Note that in practice, we first need to perform a forward DDIM step over $[t_{j},t_{j-1}]$ to obtain $\boldsymbol{z}_{j-1}$, and then we are able to perform the backward DDIM step computing $\Delta(t_{j-1}\rightarrow t_{j}|\boldsymbol{z}_{j-1})$.

\noindent\textbf{Bi-directional integration approximation (BDIA)}: 
We now present our new BDIA technique. Our primary goal is to develop an update expression for each $\boldsymbol{z}_{i-1}$ as a linear combination of $(\boldsymbol{z}_{i+1}, \boldsymbol{z}_i,\hat{\boldsymbol{\epsilon}}_{\boldsymbol{\theta}}(\boldsymbol{z}_i,i))$. As will be explained in the following, the summation of the integrations $\sum_{j=N}^{i}\int_{t_j}^{t_{j-1}}\boldsymbol{d}(\boldsymbol{z}_{\tau},\tau)d\tau$ for $\boldsymbol{z}_{i-1}$ will involve both forward DDIM updates and backward DDIM updates. 

Suppose we are at the initial time step $t_N$ with state $\boldsymbol{z}_N$. Then the next state $\boldsymbol{z}_{N-1}$ is computed by applying the forward DDIM (see (\ref{equ:ddim_f})):
\begin{align}
\boldsymbol{z}_{N-1} &= a_N\boldsymbol{z}_N +b_N\hat{\boldsymbol{\epsilon}}_{\boldsymbol{\theta}}(\boldsymbol{z}_N, N) \nonumber \\
&={\boldsymbol{z}}_N +  \Delta(t_{N}\rightarrow t_{N-1}|\boldsymbol{z}_{N}).\label{equ:BDIA_init}
\end{align}
Upon obtaining $\boldsymbol{z}_{N-1}$, we are able to compute $\Delta(t_{N-1}\rightarrow t_{N}|\boldsymbol{z}_{N-1})$ over the previous time-slot $[t_{N-1}, t_N]$ and $\Delta(t_{N-1}\rightarrow t_{N-2}|\boldsymbol{z}_{N-1})$ over the next time-slot $[t_{N-1}, t_{N-2}]$. Consequently, the integration $\int_{t_N}^{t_{N-1}}\boldsymbol{d}(\boldsymbol{z}_{\tau},\tau)d\tau$ can be  approximated by utilizing both $-\Delta(t_{N-1}\rightarrow t_{N}|\boldsymbol{z}_{N-1})$ and $\Delta(t_{N}\rightarrow t_{N-1}|\boldsymbol{z}_{N})$ as in (\ref{equ:ddim_fb}). We define the update for $\boldsymbol{z}_{i-1}$ for $i\leq N-1$ as below:

\begin{definition}
When $i\leq N-1$, let the diffusion state $\boldsymbol{z}_{i-1}$ be computed in terms of $(\boldsymbol{z}_i, \boldsymbol{z}_{i+1})$ as 
\begin{align}
\boldsymbol{z}_{i-1} &= \gamma(\boldsymbol{z}_{i+1}-\boldsymbol{z}_i) 
-\gamma\left(\frac{\boldsymbol{z}_{i}}{a_{i+1}} \hspace{-0.7mm}-\hspace{-0.7mm}\frac{b_{i+1}}{a_{i+1}}\hat{\boldsymbol{\epsilon}}_{\boldsymbol{\theta}}(\boldsymbol{z}_{i},i)-\boldsymbol{z}_i\right) + \Big[a_{i}\boldsymbol{z}_{i} \hspace{-0.7mm}+\hspace{-0.7mm} b_{i}\hat{\boldsymbol{\epsilon}}_{\boldsymbol{\theta}}(\boldsymbol{z}_{i}, i)\Big]  \label{equ:BDIA_i_ave0} \\
&=\boldsymbol{z}_{i+1}\underbrace{-(1-\gamma)(\boldsymbol{z}_{i+1}-\boldsymbol{z}_i)-\gamma\Delta(t_{i}\rightarrow t_{i+1}|\boldsymbol{z}_i)}_{\approx \int_{t_{i+1}}^{t_i}\boldsymbol{d}(\boldsymbol{z}_{\tau}, \tau)d\tau} + \underbrace{\Delta(t_{i}\rightarrow t_{i-1}|\boldsymbol{z}_i)}_{\approx \int_{t_{i}}^{t_{i-1}}\boldsymbol{d}(\boldsymbol{z}_{\tau}, \tau)d\tau} 
\label{equ:BDIA_i_ave}
\end{align}
where $\gamma\in [0,1]$. For the special case that $\gamma=1$, (\ref{equ:BDIA_i_ave}) takes a simple form 
\begin{align}
\boldsymbol{z}_{i-1}&=\boldsymbol{z}_{i+1}-\Delta(t_{i}\rightarrow t_{i+1}|\boldsymbol{z}_i) + \Delta(t_{i}\rightarrow t_{i-1}|\boldsymbol{z}_i).
\label{equ:BDIA_i}
\end{align}
\end{definition}

\begin{remark}
We note that the update expression (\ref{equ:BDIA_i_ave}) is in fact not limited to DDIM. As long as $
\Delta(t_i\rightarrow t_{i+1}|\boldsymbol{z}_i)$ and $
\Delta(t_i\rightarrow t_{i-1}|\boldsymbol{z}_i)$ represent backward and forward integration approximations conditioned on $\boldsymbol{z}_i$, (\ref{equ:BDIA_i_ave}) can then be applied to enable exact inversion. One can also easily design high-order BDIA solvers as an extension of (\ref{equ:BDIA_i_ave}). Considering the 2nd order BDIA solver,  $\boldsymbol{z}_{i-1}$ can be computed in terms of $(\boldsymbol{z}_i, \boldsymbol{z}_{i+1},\boldsymbol{z}_{i+2})$ as 
\begin{align}
&\boldsymbol{z}_{i-1} \nonumber \\
&\hspace{-0.6mm}=\hspace{-0.6mm}\boldsymbol{z}_{i+2}\underbrace{\hspace{-0.6mm}-\hspace{-0.2mm}(1\hspace{-0.8mm}-\hspace{-0.8mm}\gamma_2)(\boldsymbol{z}_{i+2} \hspace{-0.8mm}-\hspace{-0.8mm}\boldsymbol{z}_{i+1}) \hspace{-0.6mm}-\hspace{-0.6mm}\gamma_2\Delta(t_{i+1}\hspace{-1mm}\rightarrow \hspace{-0.8mm} t_{i+2}|\boldsymbol{z}_i,\boldsymbol{z}_{i+1})}_{\approx \int_{t_{i+2}}^{t_{i+1}}\boldsymbol{d}(\boldsymbol{z}_{\tau}, \hspace{-0.6mm}\tau)d\tau} \nonumber\\
&\hspace{3.5mm} \underbrace{-(1\hspace{-0.6mm}-\hspace{-0.6mm}\gamma_1)(\boldsymbol{z}_{i+1} \hspace{-0.6mm}-\hspace{-0.6mm}\boldsymbol{z}_i)\hspace{-0.6mm}-\hspace{-0.6mm}\gamma_1\Delta(t_{i}\rightarrow t_{i+1}|\boldsymbol{z}_i, \boldsymbol{z}_{i+1})}_{\approx \int_{t_{i+1}}^{t_{i}}\boldsymbol{d}(\boldsymbol{z}_{\tau}, \tau)d\tau}  \nonumber \\
&\hspace{3.5mm} + \underbrace{\Delta(t_{i}\rightarrow t_{i-1}|\boldsymbol{z}_i, \boldsymbol{z}_{i+1})}_{\approx \int_{t_{i}}^{t_{i-1}}\boldsymbol{d}(\boldsymbol{z}_{\tau}, \tau)d\tau}, 
\label{equ:BDIA_i_2nd_ave}
\end{align}
where $\gamma_1,\gamma_2\in [0,1]$. 
 It is expected that with both $\boldsymbol{z}_i$ and $\boldsymbol{z}_{i+1}$, the two integration approximations $\Delta(t_{i}\rightarrow t_{i+1}|\boldsymbol{z}_i, \boldsymbol{z}_{i+1})$ and $\Delta(t_{i+1}\rightarrow t_{i+2}|\boldsymbol{z}_i, \boldsymbol{z}_{i+1})$ would become more accurate than the ones based on  $\boldsymbol{z}_i$ or $\boldsymbol{z}_{i+1}$ alone. 
 Similarly, one can also design a 3rd order BDIA solver.
\end{remark}

It is clear from (\ref{equ:BDIA_i_ave}) that in computation of $\boldsymbol{z}_{i-1}$, the integration $\int_{t_{i+1}}^{t_i}\boldsymbol{d}(\boldsymbol{z}_{\tau},\tau)d\tau$ for the previous time-slot $[t_{i+1}, t_i]$ is approximated by taking average of the backward DDIM update $-\Delta(t_i\rightarrow t_{i+1}|\boldsymbol{z}_i)$ and the integration approximation $\boldsymbol{z}_i-\boldsymbol{z}_{i+1}$ made earlier for the same time-slot. When $\gamma=1$, the integration $\int_{t_{i+1}}^{t_i}\boldsymbol{d}(\boldsymbol{z}_{\tau},\tau)d\tau$ is simply approximated by the backward DDIM update $-\Delta(t_i\rightarrow t_{i+1}|\boldsymbol{z}_i)$.  

Next we derive the explicit expression for $\boldsymbol{z}_i$ in terms of all the historical forward and backward DDIM updates, which is summarized in a proposition below:
\begin{proposition} Let $\boldsymbol{z}_{N-1}$ and $\{\boldsymbol{z}_i| i\leq N-2\}$ be computed by following (\ref{equ:BDIA_init}) and (\ref{equ:BDIA_i_ave}) sequentially.  Then for each timestep $i\leq N-2$,  $\boldsymbol{z}_{i}$ can be represented in the form of 
\begin{align}
\boldsymbol{z}_{i}
&= \boldsymbol{z}_{N}+\sum_{j=i+2}^{N} \left[\frac{1-(-\gamma)^{j-i}}{1+\gamma}\Delta(t_{j}\rightarrow t_{j-1}|\boldsymbol{z}_{j})- \frac{\gamma+(-\gamma)^{j-i}}{1+\gamma}\Delta(t_{j-1}\rightarrow t_{j}|\boldsymbol{z}_{j-1}))\right] \label{equ:BDIA_fb_int_ave0}  \\
&\hspace{3mm}+\Delta(t_{i+1}\rightarrow t_{i}|\boldsymbol{z}_{i+1}).
\label{equ:BDIA_fb_int_ave} 
\end{align}
When $\gamma=1$, (\ref{equ:BDIA_fb_int_ave}) can be simplied to be
\begin{align}
\boldsymbol{z}_i =& \boldsymbol{z}_N + \Delta(t_{N}-t_{N-1}|\boldsymbol{z}_{N})\mathrm{mod}(N-i, 2) \nonumber \\
&+ \sum_{j=i+1}^{N-1}(-\Delta(t_j\rightarrow t_{j+1}|\boldsymbol{z}_j)+\Delta(t_j\rightarrow t_{j-1}|\boldsymbol{z}_j))\mathrm{mod}(j-i,2).
\label{equ:BDIA_fb_int}
\end{align}
\label{prop:BDIA_fb_int_ave}
\end{proposition}
\begin{proof}
See Appendix~\ref{appendix:BDIA_DDIM_ave_proof} for proof. 
\end{proof}

\begin{figure*}[t!]
\centering
\includegraphics[width=120mm]{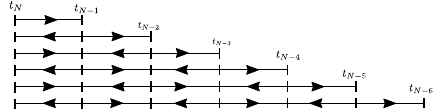}
\vspace*{-0.0cm}
\caption{\footnotesize{Schematic illustration of the BDIA-DDIM integration formulation for the special case of $\gamma=1$. Right and left arrows denote forward and backward updates, respectively. } }
\label{fig:BDIA_demo}
\vspace*{-0.3cm}
\end{figure*}

We now investigate (\ref{equ:BDIA_fb_int_ave0})-(\ref{equ:BDIA_fb_int}). It is clear from (\ref{equ:BDIA_fb_int_ave0}) that for each time-slot $[t_j, t_{j-1}]$, the summation of the two coefficients $(1-(-\gamma)^{j-i})/(1+\gamma)$ and $(\gamma+(-\gamma)^{j-i})/(1+\gamma)$ is equal to 1. This implies that the usage of BDIA-DDIM indeed leads to the averaged forward and backward DDIM updates per time-slot as we planned earlier in (\ref{equ:ddim_fb}). For the special case of $\gamma=1$, the update expression (\ref{equ:BDIA_fb_int}) becomes much simple.  Fig.~\ref{fig:BDIA_demo}  demonstrates when $\gamma=1$, how the entire integration $\int_{t_N}^{t_{i}}\boldsymbol{d}(\boldsymbol{z}_{\tau},\tau)d\tau$ for different $\boldsymbol{z}_{i}$  is approximated. It can be seen from the figure that the directions of the integration approximation for neighbouring time-slots are always opposite. In other words, the forward and backward DDIM updates are interlaced over the set of time-slots $\{(t_{j}, t_{j-1})\}_{j=N}^{i-1}$ for each $\boldsymbol{z}_{i}$. 


Next, we briefly discuss the averaging operations in EDICT and the proposed BDIA technique. It is seen from (\ref{equ:EDICT_r3})-(\ref{equ:EDICT_r4}) that EDICT performs averaging via the parameter $p$ over the two paired diffusion states $\boldsymbol{z}$ and $\boldsymbol{y}$, both of which are computed via forward DDIM updates. On the other hand, BDIA proposes to average via the parameter $\gamma$ over the DDIM forward and backward updates.  See Fig.~\ref{fig:BDIA_gamma_effect} for how the averaging operations affect the performance of EDICT and BDIA-DDIM for (round-trip) image editing.  

\begin{figure*}[t!]
\centering
\includegraphics[width=130mm]{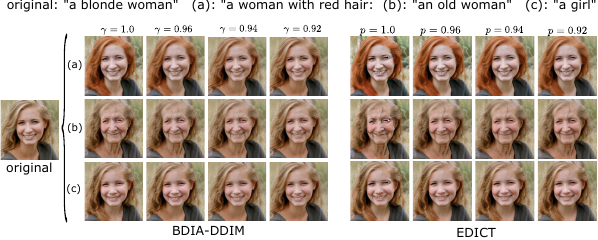}
\vspace*{-0.25cm}
\caption{\footnotesize{Demonstration of the impact of $\gamma$ and $p$ values in (round-trip) image-editing performance of BDIA-DDIM and EDICT, respectively. The number of timesteps was set to $40$ in both methods. The plots for BDIA-DDIM indicate that as $\gamma$ decreases from 1 to 0.92, the edited images tend to be visually closer to the original image. That is, the $\gamma$ parameter in BDIA-DDIM provides one more degree of freedom to allow for flexible image-editing than DDIM. 
The figure also indicates that the performance of EDICT with $p=1$ has noticeable distortions while BDIA-DDIM with $\gamma=1$ produces better image quality, and at the same time, only consumes half of the NFEs needed by EDICT. The recommended setup for $p$ in EDICT is $p=0.93$. } }
\label{fig:BDIA_gamma_effect}
\vspace*{-0.3cm}
\end{figure*}

\noindent \textbf{BDIA-DDIM inversion}: Whereas the conventional DDIM inversion \eqref{equ:DDIM_inv2} requires the approximation $\boldsymbol{z}_{i-1} \approx \boldsymbol{z}_i$, which is only true in the limit of infinite steps,
the formulation \eqref{equ:BDIA_i_ave} allows exact inversion (up to floating point error). That is, it follows from (\ref{equ:BDIA_i_ave}) that the diffusion state $\boldsymbol{z}_{i+1}$ can be computed in terms of $(\boldsymbol{z}_i, \boldsymbol{z}_{i-1})$ as 
\begin{align}
\boldsymbol{z}_{i+1} 
& =
\boldsymbol{z}_{i-1}/\gamma - \Big[a_{i}\boldsymbol{z}_{i} \hspace{-0.7mm}+\hspace{-0.7mm} b_{i}\hat{\boldsymbol{\epsilon}}_{\boldsymbol{\theta}}(\boldsymbol{z}_{i}, i)\Big]/\gamma+ \left(\frac{\boldsymbol{z}_{i}}{a_{i+1}} \hspace{-0.7mm}-\hspace{-0.7mm}\frac{b_{i+1}}{a_{i+1}}\hat{\boldsymbol{\epsilon}}_{\boldsymbol{\theta}}(\boldsymbol{z}_{i},i)\right). \label{equ:BDIA_i_inv}
\end{align}
Similarly to the computation (\ref{equ:BDIA_i_inv}), EDICT also does not involve any approximation and results in exact diffusion inversion. 
However, in contrast to EDICT, (\ref{equ:BDIA_i_inv}) does not require a doubling of the NFE.


For the special case of $\gamma=1$, (\ref{equ:BDIA_i}) or (\ref{equ:BDIA_i_inv}) is symmetric in time. That is, switching the timestep $t_{i+1}$ and $t_{i-1}$ in (\ref{equ:BDIA_i}) or (\ref{equ:BDIA_i_inv}) with $\gamma=1$ inverts the diffusion direction. We summarize the above property of time-symmetry in a lemma below:

\begin{proposition}[time-symmetry]
Switching the timestep $t_{i-1}$ and $t_{i+1}$ in (\ref{equ:BDIA_i}) produces the reverse update (\ref{equ:BDIA_i_inv}) under the setup $\gamma=1$, and vice versa.
\end{proposition}

\begin{remark}
We have also designed coupled BDIA (CBDIA) as an extension of EDICT. CBDIA is also invertible and needs to perform two NFEs per timestep in general. It is found that CBDIA includes BDIA with $\gamma=1$ as a special case. See Appendix~\ref{appendix:CBDIA} for details.  
\end{remark}

\section{BDIA for EDM sampling procedure}
In this section, we explain how to apply our new technique BDIA to the EDM sampling procedure \cite{Karras22EDM}. We emphasize that the new method BDIA-EDM is designed to improve the sampling quality instead of enabling inversion of EDM. 

In brief, the recent work \cite{Karras22EDM} reparameterizes the forward diffusion process (\ref{equ:forwardGaussian}) to be
\begin{align}
q_{t|0}(\boldsymbol{z}_t|\boldsymbol{x}) = \mathcal{N}(\boldsymbol{z}_t|\alpha_t\boldsymbol{x}, \alpha_t^2\tilde{\sigma}_t^2\boldsymbol{I}),
\end{align}
where $\sigma_t$ of (\ref{equ:forwardGaussian}) is represented as $\sigma_t=\alpha_t\tilde{\sigma}_t$. 
The EDM sampling procedure is then designed in \cite{Karras22EDM} to solve the corresponding ODE using improved Euler method. In particular, the diffusion state $\boldsymbol{z}_{i-1}$ at timestep $t_{i-1}$ given $\boldsymbol{z}_i$ is computed as
\begin{align}
\tilde{\boldsymbol{z}}_{i-1} &= \textcolor{blue}{\boldsymbol{z}_i} + (t_{i-1}-t_i) \boldsymbol{d}_i \label{equ:ImpEuler1}  \\
\boldsymbol{z}_{i-1} &= \textcolor{blue}{\boldsymbol{z}_i} + (t_{i-1}-t_i)(\frac{1}{2}\boldsymbol{d}_i+\frac{1}{2}\boldsymbol{d}_{i-1|i}') 
 \label{equ:ImpEuler2}
\end{align}
where $\boldsymbol{d}_i=\boldsymbol{d}(\boldsymbol{z}_i,t_i)$ and $\boldsymbol{d}_{i-1|i}'=\boldsymbol{d}(\tilde{\boldsymbol{z}}_{i-1},t_{i-1})$. $\tilde{\boldsymbol{z}}_{i-1}$ is the intermediate estimate of the diffusion state $\boldsymbol{z}$ at time $t_{i-1}$. $\boldsymbol{z}_{i-1}$ is computed by utilizing the average of the two gradients $\boldsymbol{d}_i$ and $\boldsymbol{d}_{i-1|i}'$.  

Next we consider refining the estimate for $\boldsymbol{z}_i$ (highlighted in \textcolor{blue}{blue} in (\ref{equ:ImpEuler1})-(\ref{equ:ImpEuler2})) via BDIA when computing $\boldsymbol{z}_{i-1}$. We observe that in the iterative updates of EDM, both $\boldsymbol{d}_{i|i+1}'=\boldsymbol{d}(\tilde{\boldsymbol{z}}_{i},t_i)$ and $\boldsymbol{d}_i=\boldsymbol{d}(\boldsymbol{z}_i,t_i)$ are evaluated at the same timepoint $t_i$. Specifically, $\boldsymbol{d}_{i|i+1}'$ is for approximating $\int_{t_{i+1}}^{t_{i}}\boldsymbol{d}(\boldsymbol{z}_{\tau},\tau)d\tau$ over the previous timeslot $[t_{i+1}, t_i]$ while $\boldsymbol{d}_{i}$ is for approximating $\int_{t_{i}}^{t_{i-1}}\boldsymbol{d}(\boldsymbol{z}_{\tau},\tau)d\tau$ over the current timeslot $[t_i, t_{i-1}]$. It is thus natural to reuse $\boldsymbol{d}_{i}$ for better approximating the previous integration $\int_{t_{i+1}}^{t_{i}}\boldsymbol{d}(\boldsymbol{z}_{\tau},\tau)d\tau$. Similarly to BDIA-DDIM, the new estimate $\hat{\boldsymbol{z}}_i$ is computed as
\begin{align}
\hat{\boldsymbol{z}}_i = \boldsymbol{z}_{i+1} + (1-\gamma) (\boldsymbol{z}_i-\boldsymbol{z}_{i+1})+\gamma(t_i-t_{i+1})(\frac{1}{2}\boldsymbol{d}_{i+1}+\frac{1}{2}\boldsymbol{d}_{i}),
\end{align}
where $\gamma\in [0,1]$. Once $\hat{\boldsymbol{z}}_i$ is obtained, it is then employed to replace $\boldsymbol{z}_i$ in \textcolor{blue}{blue} in (\ref{equ:ImpEuler1})-(\ref{equ:ImpEuler2}). We summarize the update procedure of BDIA-EDM in Alg.~\ref{alg:BDIAEDM}. 

\begin{algorithm}[t!]
   \caption{\small BDIA-EDM to improve sampling quality }
   \label{alg:BDIAEDM}
\begin{algorithmic}[1]
\STATE {\small {\bfseries Input:}  \hspace{-1.5mm}  $\begin{array}{l} \textrm{number of time steps } N,\end{array}$ $\gamma\in [0,1]$ }
   \STATE {\small {\bfseries Sample} $\boldsymbol{z}_N\sim\mathcal{N}(\boldsymbol{0}, \alpha_{t_N}^2\tilde{\sigma}_{t_N}^2\boldsymbol{I})$  }
   \FOR{\small $i\in\{N, N{-}1, \ldots, 1\}$}
   \STATE \hspace{-0mm}{\small $\boldsymbol{d}_i=\boldsymbol{d}(\boldsymbol{z}_i,t_i)$} 
   \IF{$i<N$} 
   \STATE $\hat{\boldsymbol{z}}_i \!=\! \boldsymbol{z}_{i+1} +(1-\gamma)(\boldsymbol{z}_i \!-\!\boldsymbol{z}_{i+1})  \!+\! \gamma(t_{i}-t_{i+1})(\frac{1}{2}\boldsymbol{d}_{i+1}\!+\!\frac{1}{2}\boldsymbol{d}_i) $ \textcolor{blue}{[averaged integration approx.]}  
   \ELSE  
   \STATE $\hat{\boldsymbol{z}}_i=\boldsymbol{z}_i$
   \ENDIF
   \STATE $\tilde{\boldsymbol{z}}_{i-1} \leftarrow \hat{\boldsymbol{z}}_i + (t_{i-1}-t_i)\boldsymbol{d}_i$
   \IF{ $\sigma_{t_{i-1}}\neq 0$ }
   \STATE 
   $\boldsymbol{d}_{i-1|i}'=\boldsymbol{d}(\tilde{\boldsymbol{z}}_{i-1},t_{i-1})$
   \STATE $\boldsymbol{z}_{i-1}\leftarrow \hat{\boldsymbol{z}}_i + (t_{i-1}-t_i)\left(\frac{1}{2}\boldsymbol{d}_{i} +\frac{1}{2}\boldsymbol{d}_{i-1|i}'\right)$  
   \ENDIF
   \ENDFOR 
   \STATE {\bfseries Output:} {\small $\boldsymbol{z}_{0}$  }\vspace{1mm}
\end{algorithmic}
\end{algorithm}

\begin{remark}
Inspired by the design of BDIA-EDM, we have also proposed BDIA-DPM-Solver++ as an extension of DPM-Solver++. See Appendix~\ref{appendix:BDIA-DPM-Solver} for details.   
\end{remark}

\section{Related Work}
\label{sec:related_works}

In the numerical integration literature, there is a branch of research on development of time-reversible ODE solvers. For instance, Verlet integration is a time-reversible method for solving  2nd-order ODEs \cite{Verlet67VerletInt}. Leapfrog integration is another time-reversible method also developed for solving 2nd-order ODEs \cite{Skeel93leapfrog}.

\section{Experiments}
\label{sec:exp}

We conducted two types of experiments: (1) evaluation of image sampling for both BDIA-DDIM and BDIA-EDM; (2) image-editing via BDIA-DDIM. It was found that our new technique BDIA produces promising results for both tasks.

\subsection{Evaluation of image sampling}
\label{subsec:exp_imgSampling}
\noindent \textbf{Text-to-image generation:} In this task, we tested three sampling methods by using  StableDiffusion V2\footnote{\url{https://github.com/Stability-AI/stablediffusion}}, which are DDIM, DPM-Solver, and BDIA-DDIM. The parameter $\gamma$ in BDIA-DDIM was set to $\gamma=0.5$. COCO2014 validation set was utlized for this task. For each method, 20K images of size $512\times 512$ were generated by feeding 20K text prompts the diffusion model. All the three methods share the same seed of the random noise generator and the same text prompts. The FID score for each method was computed by resizing the generated images to size of $256\times 256$ due to the fact that the original images in COCO2014 are in general have lower resolution than the size of $512\times 512$.

\begin{table}[h!]
\caption{ Comparison of three methods for the task of text-to-image generation with 10 timesteps  \vspace{-2mm} } 
\vspace*{0.2cm}
\label{tab:t2i_FID_compare}
\centering
\begin{tabular}{|c|c|c|c|}
\hline
 { sampling methods } & \hspace{0mm} {{ BDIA-DDIM }}\hspace{0mm}
&   \hspace{0mm} {{ DDIM }}\hspace{0mm}  
&   \hspace{0mm} {{ DPM-Solver }}\hspace{0mm}  
\\ \hline
{FID}  & \textbf{{12.62}} & {15.04} &  {16.06} 
\\ \hline
\end{tabular}
\end{table}

It is clear from Table~\ref{tab:t2i_FID_compare} that BDIA-DDIM has a big FID performance gain than the other two methods. This indicates that the introduced backward DDIM update per timestep helps to improve the accuracy of the coresponding integration approximation. The visual improvement of BDIA-DDIM over DDIM is also demonstrated in Fig.~\ref{fig:BDIADDIM_t2i}. 

\noindent \textbf{Classical image generation:} In this experiment, we consider the task of noise-to-image generation without input texts. The parameter $\gamma$ in BDIA-DDIM and BDIA-EDM was set to $\gamma=1.0$.  The tested pre-trained models can be found in Appendix~\ref{appendix:pre_trained_models}. Given a pre-trained model, 50K artificial images were generated for a particular NFE, and the corresponding FID score was computed.   

Table~\ref{tab:BDIA_DDIM} and \ref{tab:BDIA_EDM} summarize the computed FID scores. It is clear that by incorporating BDIA into both DDIM and EDM, the FID scores are improved consierably. This can be explained by the fact that BDIA introduces the additional backward integration approximation per time-step in the sampling process. This makes the resulting final integration approximation more accurate.

\begin{table}[h!]
\caption{ FID comparison for unconditional sampling.  } 
\vspace*{-0.2cm}
\label{tab:BDIA_DDIM}
\centering
\begin{tabular}{|c|c|c|c||c|c|c|}
\hline
\hspace{-2.5mm} {\scriptsize  timesteps }
\hspace{-2.5mm} & \hspace{-1.5mm}  
\hspace{-2.5mm} & \hspace{-2.5mm} {{\scriptsize DDIM}}
\hspace{-2.5mm} &   \hspace{-2.5mm} 
{{\scriptsize  $\begin{array}{c}\textrm{BDIA-DDIM}  \end{array}$}}
\hspace{-2.5mm} & \hspace{-1.5mm} 
\hspace{-2.5mm} & \hspace{-2.5mm} 
{{\scriptsize DDIM}}
\hspace{-2.5mm} & \hspace{-3mm} 
{{\scriptsize $\begin{array}{c}\textrm{BDIA-DDIM} \end{array}$}}
\hspace{-3mm}   
\\ \hline
\hspace{-2.5mm}
{\scriptsize  10}  
\hspace{-2.5mm} & \hspace{-1.5mm} 
\multirow{3}{*}{\rotatebox{90}{\scriptsize CIFAR10}}  
\hspace{-1.5mm} & \hspace{-2.5mm} 
{\scriptsize  14.38} 
\hspace{-2.5mm} & \hspace{-2.5mm}
{\scriptsize \textbf{10.03}} 
\hspace{-2.5mm} & \hspace{-1.5mm}
\multirow{3}{*}{\rotatebox{90}{\scriptsize Celeba}}
\hspace{-1.5mm} & \hspace{-2.5mm}
{\scriptsize  13.41 } 
\hspace{-2.5mm}  & \hspace{-3mm}
{\scriptsize  \textbf{10.86} } \hspace{-3mm}
\\ \cline{1-1}  \cline{3-4}  \cline{6-7}
\hspace{-2.5mm}
{\scriptsize  20} 
\hspace{-2.5mm} & \hspace{-1.5mm}
\hspace{-1.5mm} & \hspace{-2.5mm} 
{\scriptsize  7.51} 
\hspace{-2.5mm} & \hspace{-2.5mm}
{\scriptsize 
 \textbf{6.29}} 
 \hspace{-2.5mm} & \hspace{-1.5mm} 
 \hspace{-1.5mm} & \hspace{-2.5mm} 
 {\scriptsize  9.45} 
 \hspace{-2.5mm}  & \hspace{-3mm}
 {\scriptsize  \textbf{8.86}} 
 \hspace{-3mm}
\\ \cline{1-1} \cline{3-4}  \cline{6-7}
\hspace{-2.5mm}
{\scriptsize  40}  
\hspace{-2.5mm} & \hspace{-1.5mm}
\hspace{-1.5mm} & \hspace{-2.5mm} {\scriptsize  4.95} 
\hspace{-2.5mm} & \hspace{-2.5mm}
{\scriptsize 
 \textbf{4.63}}
\hspace{-2.5mm} & \hspace{-1.5mm}
\hspace{-1.5mm} & \hspace{-2.5mm}
 {\scriptsize  6.93 } 
 \hspace{-2.5mm} & \hspace{-3mm}  {\scriptsize  \textbf{6.50} } \hspace{-3mm}
 \\ \hline
\end{tabular}
\vspace{-3mm}
\end{table}

\begin{table}[h!]
\caption{ \footnotesize{FID comparison of EDM and BDIA-EDM. The reason for testing smaller NFEs for CIFAR10 than those for other datasets is because in \cite{Karras22EDM}, the optimal NFE for CIFAR10 is smaller than others.   }} 
\vspace*{-0.2cm}
\label{tab:BDIA_EDM}
\centering
\begin{tabular}{|c|c|c|c|c|}
\cline{2-5}
\multicolumn{1}{c|}{} & \hspace{0mm} {{\scriptsize CIFAR10 }}\hspace{0mm}
&   \hspace{0mm} {{\scriptsize FFHQ}}\hspace{0mm} 
&   \hspace{0mm} {{\scriptsize AFHQV2}}\hspace{0mm} 
&   \hspace{0mm} {{\scriptsize ImageNet64}}\hspace{0mm} 
\\ \hline
{\scriptsize NFEs} & {\scriptsize 35 }  & {\scriptsize 39} & {\scriptsize {39}}
& {\scriptsize {39}}
\\ \hline
{\scriptsize EDM} & {\scriptsize 1.85}  & {\scriptsize 2.64} & {\scriptsize 2.08} & {\scriptsize {2.51}}
\\ \hline
{\scriptsize BDIA-EDM} & {\scriptsize \textbf{1.79}}  & {\scriptsize \textbf{2.54}} & {\scriptsize \textbf{2.02}}
& {\scriptsize \textbf{2.38}}
\\ \hline \hline
{\scriptsize NFEs} & {\scriptsize 43 }  & {\scriptsize 47} & {\scriptsize {47}}
& {\scriptsize {47}}
\\ \hline
{\scriptsize EDM} & {\scriptsize 1.85}  & {\scriptsize 2.55} & {\scriptsize 2.06} & {\scriptsize {2.44}}
\\ \hline
{\scriptsize BDIA-EDM} & {\scriptsize \textbf{1.80}}  & {\scriptsize \textbf{2.46}} & {\scriptsize \textbf{2.01}}
& {\scriptsize \textbf{2.35}}
\\ \hline
\end{tabular}
\vspace*{-0.3cm}
\end{table}

\subsection{Evaluation of image-editing}
In this second experiment, we evaluated BDIA-DDIM for text-based  and ControlNet-based image-editing by utilizing the open-source repository of EDICT\footnote{\url{https://github.com/salesforce/EDICT}} and ControlNet. Fig.~\ref{fig:image_editing_controlnet} and\ref{fig:image_editing_mixing} visualize the obtained results. We point out that for text-based image editing, BDIA-DDIM produces comparable results to EDICT while reducing by approximately half the NFE compared to EDICT.


\begin{figure*}[h!]
\centering
\includegraphics[width=120mm]
{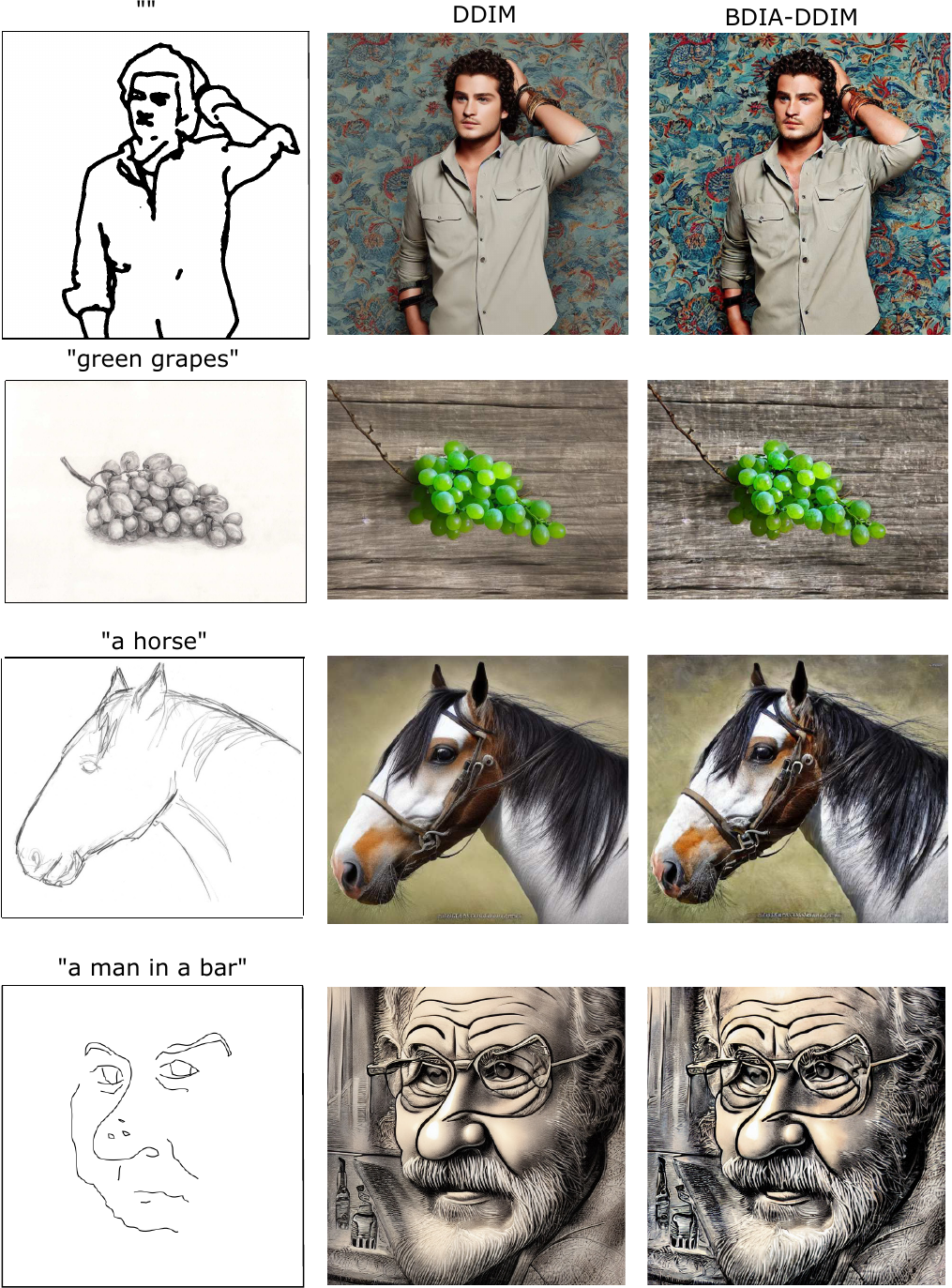}
\vspace*{-0.2cm}
\caption{ControlNet-based image editing by using BDIA-DDIM and DDIM at 10 timesteps.  The hyper-parameter $\gamma$ in BDIA-DDIM was set to $0.5$.   
}
\label{fig:image_editing_controlnet}
\vspace*{-0.0cm}
\end{figure*}
 
 \begin{figure*}[h!]
\centering
\includegraphics[width=120mm]{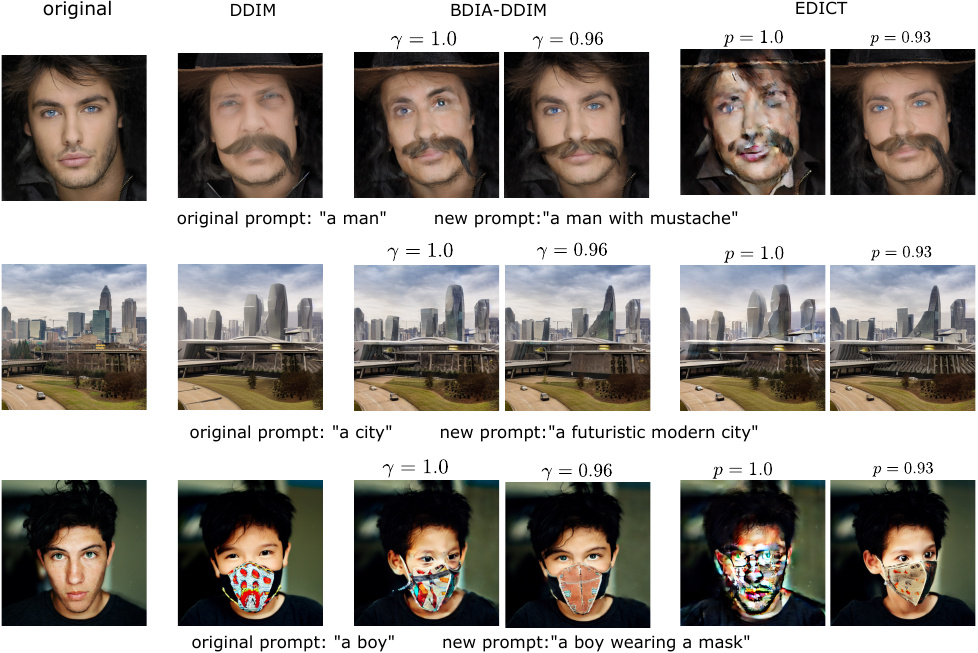}
\vspace*{-0.2cm}
\caption{Round-trip image editing by using BDIA-DDIM, EDICT, and DDIM. The value $p=0.93$ for EDICT is the recommended setup in \cite{Wallace23EDICT}. The number of timesteps was set to 40 and the classifier guidance scale was 4.0. The $\gamma$ parameter in BDIA-DDIM can be adjusted by the artist to produce different quality and effects. BDIA-DDIM only needs to perform half of NFEs needed in EDICT for each image generation.  
Please enlarge to see details.
}
\label{fig:image_editing_mixing}
\vspace*{-0.0cm}
\end{figure*}

\vspace{-2mm}
\section{Conclusions}
\vspace{-2mm}
In this paper, we have proposed a new technique BDIA, to assist DDIM in achieving exact diffusion inversion. The key step of BDIA-DDIM is to perform DDIM update procedure twice at each time step $t_i$: one over the previous time-slot $[t_i, t_{i+1}]$ and the other over the next time-slot $[t_i,t_{i-1}]$ in computing $\boldsymbol{z}_{i-1}$. By doing so, the expression for $\boldsymbol{z}_{i-1}$ becomes a linear combination of $(\boldsymbol{z}_i, \hat{\boldsymbol{\epsilon}}_{\boldsymbol{\theta}}(\boldsymbol{z}_i,i), \boldsymbol{z}_{i+1})$ that is symmetric in time. As a result, $\boldsymbol{z}_{i+1}$ can be computed exactly as a linear function of $(\boldsymbol{z}_i, \hat{\boldsymbol{\epsilon}}_{\boldsymbol{\theta}}(\boldsymbol{z}_i,i), \boldsymbol{z}_{i-1})$, enabling exact diffusion inversion. Note that although the DDIM update is evaluated twice at each step, this is inexpensive since the costly neural functional evaluation is performed only once.


\begin{thebibliography}{10}

\bibitem{Arjovsky17WGAN}
M.~Arjovsky, S.~Chintala, and L.~Bottou.
\newblock {Wasserstein GAN}.
\newblock arXiv:1701.07875 [stat.ML], 2017.

\bibitem{Bao22DPM_cov}
F.~Bao, C.~Li, J.~Sun, J.~Zhu, and B.~Zhang.
\newblock {Estimating the Optimal Covariance with Imperfect Mean in Diffusion
  Probabilistic Models}.
\newblock In {\em ICML}, 2022.

\bibitem{Bao22DPM}
F.~Bao, C.~Li, J.~Zhu, and B.~Zhang.
\newblock {Analytic-DPM: an Analytic Estimate of the Optimal Reverse Variance
  in Diffusion Probabilistic Models}.
\newblock In {\em ICLR}, 2022.

\bibitem{Bishop06}
C.~M. Bishop.
\newblock {\em {Pattern Recognition and Machine Learning}}.
\newblock Springer, 2006.

\bibitem{Chen20WaveGrad}
N.~Chen, Y.~Zhang, H.~Zen, R.~J. Weiss, M.~Norouzi, and W.~Chan.
\newblock {WaveGrad: Estimating Gradients for Waveform Generation}.
\newblock arXiv:2009.00713, September 2020.

\bibitem{Dhariwal21DPM}
P.~Dhariwal and A.~Nichol.
\newblock Diffusion models beat gans on image synthesis.
\newblock arXiv:2105.05233 [cs.LG], 2021.

\bibitem{Dinh14Nice}
L.~Dinh, D.~Krueger, and Y.~Bengio.
\newblock {Nice: Non-linear independent components estimation}.
\newblock arXiv preprint arXiv:1410.8516, 2014.

\bibitem{Dinh16DensityEsti}
L.~Dinh, J.~Sohl-Dickstein, and S.~Bengio.
\newblock {Density estimation using real nvp}.
\newblock arXiv preprint arXiv:1605.08803, 2016.

\bibitem{Goodfellow14GAN}
I.~Goodfellow, J.~Pouget-Abadie, M.~Mirza, B.~Xu, D.~Warde-Farley, S.~Ozair,
  A.~Courville, and Y.~Bengio.
\newblock {Generative Adversarial Nets}.
\newblock In {\em Proceedings of the International Conference on Neural
  Information Processing Systems}, pages 2672--2680, 2014.

\bibitem{Gulrajani17WGANGP}
I.~Gulrajani, F.~Ahmed, M.~Arjovsky, V.~Dumoulin, and A.~C. Courville.
\newblock {Improved training of wasserstein gans}.
\newblock In {\em Advances in neural information processing systems}, pages
  5767--5777, 2017.

\bibitem{Ho20DDPM}
J.~Ho, A.~Jain, and P.~Abbeel.
\newblock Denoising diffusion probabilistic models.
\newblock In {\em NeurIPS}, 2020.

\bibitem{Ho22ClassiferFreeGuide}
J.~Ho and T.~Salimans.
\newblock {Classifier-free diffusion guidance}.
\newblock arXiv preprint arXiv:2207.12598, 2022.

\bibitem{Huberman23DDPMInversion}
I.~Huberman-Spiegelglas, V.~Kulikov, and T.~Michaeli.
\newblock {An Edit Friendly DDPM Noise Space: Inversion and Manipulations}.
\newblock arXiv:2304.06140v2 [cs.CV], 2023.

\bibitem{Hyvarinen05ScoreMatching}
A.~Hyvarinen.
\newblock {Estimation of non-normalized statistical models by score matching}.
\newblock {\em Journal of Machine Learning Research}, 24:695--709, 2005.

\bibitem{Karras22EDM}
T.~Karras, M.~Aittala, T.~Alia, and S.~Laine.
\newblock {Elucidating the Design Space of Diffusion-Based Generative Models}.
\newblock In {\em 36th Conference on Nueral Information Processing Systems
  (NeurIPS)}, 2022.

\bibitem{Kim22GuidedDiffusion}
D.~Kim, Y.~Kim, S.~J. Kwon, W.~Kang, and I.-C. Moon.
\newblock {Refining Generative Process with Discriminator Guidance in
  Score-based Diffusion Models}.
\newblock arXiv preprint arXiv:2211.17091 [cs.CV], 2022.

\bibitem{Kingma18Glow}
D.~P. Kingma and P.~Dhariwal.
\newblock {Glow: Generative flow with invertible 1x1 convolutions}.
\newblock In {\em Advances in neural information processing systems}, 2018.

\bibitem{Kingma21DDPM}
D.~P. Kingma, T.~Salimans, B.~Poole, and J.~Ho.
\newblock Variational diffusion models.
\newblock arXiv: preprint arXiv:2107.00630, 2021.

\bibitem{Lam22BDDM}
M.~W.~Y. Lam, J.~Wang, D.~Su, and D.~Yu.
\newblock {BDDM: Bilateral Denoising Diffusion Models for Fast and High-Quality
  Speech Synthesis}.
\newblock In {\em ICLR}, 2022.

\bibitem{Liu22PNDM}
L.~Liu, Y.~Ren, Z.~Lin, and Z.~Zhao.
\newblock {Pseudo Numerical Methods for Diffusion Models on Manifolds}.
\newblock In {\em ICLR}, 2022.

\bibitem{Lu22DPM_Solver}
C.~Lu, Y.~Zhou, F.~Bao, J.~Chen, C.~Li, and J.~Zhu.
\newblock {DPM-Solver: A Fast ODE Solver for Diffusion Probabilistic Sampling
  in Around 10 Steps}.
\newblock In {\em NeurIPS}, 2022.

\bibitem{Mokady23NullTestInv}
R.~Mokady, A.~Hertz, K.~Aberman, Y.~Pritch, and D.~Cohen-Or.
\newblock {Null-text Inversion for Editing Real Images using Guided Diffusion
  Models}.
\newblock In {\em CVPR}, 2023.

\bibitem{Nichol21DDPM}
A.~Nichol and P.~Dhariwal.
\newblock Improved denoising diffusion probabilistic models.
\newblock arXiv preprint arXiv:2102.09672, 2021.

\bibitem{Nichol22GLIDE}
A.~Nichol, P.~Dharwal, A.~Ramesh, P.~Shyam, P.~Mishkin, B.~McGrew,
  I.~Sutskever, and M.~Chen.
\newblock {GLIDE: Towards Photorealistic image generation and editing with
  text-guided diffusion models}.
\newblock In {\em ICML}, 2022.

\bibitem{Rombach22LDM}
R.~Rombach, A.~Blattmann, D.~Lorenz, P.~Esser, and B.~Ommer.
\newblock {High-resolution image synthesis with latent diffusion models}.
\newblock In {\em CVPR}, 2022.

\bibitem{Rombach22StableDiffusion}
R.~Rombach, A.~Blattmann, D.~Lorenz, P.~Esser, and B.~Ommer.
\newblock {On High-resolution image synthesis with latent diffusion models}.
\newblock In {\em CVPR}, page 10684–10695, 2022.

\bibitem{Ronneberger15Unet}
O.~Ronneberger, P.~Fischer, and T.~Brox.
\newblock {U-Net: Convolutional Networks for Biomedical Image Segmentation}.
\newblock arXiv:1505.04597 [cs.CV], 2015.

\bibitem{Saharia22Imagen}
C.~Saharia, W.~Chan, S.~Saxena, L.~Li, J.~Whang, E.~Denton, S.-K.-S.
  Ghasemipour, B.-K. Ayan, S.~S. Mahdavi, R.-G. Lopes, T.~Salimans, J.~Ho,
  D.~J. Fleet, and M.~Norouzi.
\newblock {Photorealistic text-to-image diffusion models with deep language
  understanding}.
\newblock arXiv preprint arXiv:2205.11487, 2022.

\bibitem{Sauer22StyleGAN}
A.~Sauer, K.~Schwarz, and A.~Geiger.
\newblock {StyleGAN-XL: Scaling StyleGAN to large diverse datasets}.
\newblock In {\em SIGGRAPH}, 2022.

\bibitem{Shi23DragDiffusion}
Y.~Shi, C.~Xue, J.~Pan, and W.~Zhang.
\newblock {DragDiffusion: Harnessing Diffusion Models for Interactive
  Point-based Image Editing}.
\newblock arXiv:2306.14435v2, 2023.

\bibitem{Dickstein15DPM}
J.~Sohl-Dickstein, E.~Weiss, N.~Maheswaranathan, and S.~Ganguli.
\newblock Deep unsupervised learning using nonequilibrium thermodynamics.
\newblock ICML, 2015.

\bibitem{Song21DDIM}
J.~Song, C.~Meng, and S.~Ermon.
\newblock {Denoising Diffusion Implicit Models}.
\newblock In {\em ICLR}, 2021.

\bibitem{Song21DPM}
Y.~Song, C.~Durkan, I.~Murray, and S.~Ermon.
\newblock Maximum likelihood training of score-based diffusion models.
\newblock In {\em Advances in neural information processing systems (NeurIPS)},
  2021.

\bibitem{Song19}
Y.~Song and S.~Ermon.
\newblock {Generative modeling by estimating gradients of the data
  distribution}.
\newblock In {\em Advances in neural information processing systems (NeurIPS)},
  page 11895–11907, 2019.

\bibitem{Song21SDE_gen}
Y.~Song, J.~S.-Dickstein, D.~P. Kingma, A.~Kumar, S.~Ermon, and B.~Poole.
\newblock {Score-Based Generative Modeling Through Stochastic Differential
  Equations}.
\newblock In {\em ICLR}, 2021.

\bibitem{Wallace23EDICT}
B.~Wallace, A.~Gokul, and N.~Naik.
\newblock {EDICT: Exact Diffusion Inversion via Coupled Transformations}.
\newblock In {\em CVPR}, 2023.

\bibitem{Verlet67VerletInt}
L.~Verlet.
\newblock {Computer Experiments on Classical Fluids. I. Thermodynamical Properties of Lennard-Jones Molecules}.
\newblock {\em Physical Review}, 159:98--103, 1967.

\bibitem{Skeel93leapfrog}
R.~D.~Skeel.
\newblock {Variable Step Size Destabilizes the Stamer/Leapfrog/Verlet Method}.
\newblock {\em BIT Numerical Mathematics}, 33:172--175, 1993.


\bibitem{GuoqiangIIA23}
G.~Zhang, K.~Niwa, and W.~B. Kleijn.
\newblock {On Accelerating Diffusion-Based Sampling Processes by Improved
  Integration Approximation}.
\newblock arXiv:2304.11328 [cs.LG], 2023.

\bibitem{Zhang22DEIS}
Q.~Zhang and Y.~Chenu.
\newblock {Fast Sampling of Diffusion Models with Exponential Integrator}.
\newblock arXiv:2204.13902 [cs.LG], 2022.

\end{thebibliography}


\newpage

\appendix

\appendix
\onecolumn


\section{Proof for Proposition~\ref{prop:BDIA_fb_int_ave} }
\label{appendix:BDIA_DDIM_ave_proof}

We prove Proposition~\ref{prop:BDIA_fb_int_ave} by induction. To do so, we first compute the update expression for $\boldsymbol{z}_{N-2}$. It is known from (\ref{equ:BDIA_init}) that 
\begin{align}
\boldsymbol{z}_{N-1} &= \boldsymbol{z}_N + \Delta(t_{N}\rightarrow t_{N-1}|\boldsymbol{z}_{N}).  \nonumber
\end{align}
With the expression for $\boldsymbol{z}_{N-1}$, $\boldsymbol{z}_{N-2}$ can then be computed by using (\ref{equ:BDIA_i_ave}):
\begin{align}
\boldsymbol{z}_{N-2}&=\boldsymbol{z}_{N}-(1-\gamma)(\boldsymbol{z}_{N}-\boldsymbol{z}_{N-1})-\gamma\Delta(t_{N-1}\rightarrow t_{N}|\boldsymbol{z}_{N-1}) + \Delta(t_{N-1}\rightarrow t_{N-2}|\boldsymbol{z}_{N-1}) \nonumber\\
&=\boldsymbol{z}_{N}+(1-\gamma)\Delta(t_{N}\rightarrow t_{N-1}|\boldsymbol{z}_{N})-\gamma\Delta(t_{N-1}\rightarrow t_{N}|\boldsymbol{z}_{N-1}) + \Delta(t_{N-1}\rightarrow t_{N-2}|\boldsymbol{z}_{N-1}) \nonumber
\end{align}
which satisfies (\ref{equ:BDIA_fb_int_ave}). Furthermore, the difference $\boldsymbol{z}_{N-1}-\boldsymbol{z}_{N-2}$ can be represented as 
\begin{align}
\boldsymbol{z}_{N-1}-\boldsymbol{z}_{N-2}
= \gamma\Delta(t_{N}\rightarrow t_{N-1}|\boldsymbol{z}_{N})+\gamma\Delta(t_{N-1}\rightarrow t_{N}|\boldsymbol{z}_{N-1}) - \Delta(t_{N-1}\rightarrow t_{N-2}|\boldsymbol{z}_{N-1}) 
\label{equ:z_N_difference_proof}
\end{align}

Next, we show that the expression for $\boldsymbol{z}_{N-3}$ also takes the form of (\ref{equ:BDIA_fb_int_ave}). Again by using (\ref{equ:BDIA_i_ave}), $\boldsymbol{z}_{N-3}$ can be computed to be
\begin{align}
\boldsymbol{z}_{N-3}&=\boldsymbol{z}_{N-1}-(1-\gamma)(\boldsymbol{z}_{N-1}-\boldsymbol{z}_{N-2})-\gamma\Delta(t_{N-2}\rightarrow t_{N-1}|\boldsymbol{z}_{N-2}) + \Delta(t_{N-2}\rightarrow t_{N-3}|\boldsymbol{z}_{N-2}) \nonumber\\
&\stackrel{(a)}{=}\boldsymbol{z}_{N-1}-(1-\gamma)\Big[\gamma\Delta(t_{N}\rightarrow t_{N-1}|\boldsymbol{z}_{N})+\gamma\Delta(t_{N-1}\rightarrow t_{N}|\boldsymbol{z}_{N-1}) - \Delta(t_{N-1}\rightarrow t_{N-2}|\boldsymbol{z}_{N-1})\Big]\nonumber\\
&\hspace{3mm}-\gamma\Delta(t_{N-2}\rightarrow t_{N-1}|\boldsymbol{z}_{N-2}) + \Delta(t_{N-2}\rightarrow t_{N-3}|\boldsymbol{z}_{N-2}) \nonumber \\
&=\boldsymbol{z}_{N}+[1-(1-\gamma)\gamma]\Delta(t_{N}\rightarrow t_{N-1}|\boldsymbol{z}_{N})-(1-\gamma)\gamma\Delta(t_{N-1}\rightarrow t_{N}|\boldsymbol{z}_{N-1}) \nonumber\\
&\hspace{3mm}+(1-\gamma)\Delta(t_{N-1}\rightarrow t_{N-2}|\boldsymbol{z}_{N-1})-\gamma\Delta(t_{N-2}\rightarrow t_{N-1}|\boldsymbol{z}_{N-2}) \nonumber\\
&\hspace{3mm}+ \Delta(t_{N-2}\rightarrow t_{N-3}|\boldsymbol{z}_{N-2}) \nonumber \\
&=\boldsymbol{z}_{N}+\frac{1-(-\gamma)^3}{1+\gamma} \Delta(t_{N}\rightarrow t_{N-1}|\boldsymbol{z}_{N})-\frac{\gamma(1-\gamma^2)}{1+\gamma}\Delta(t_{N-1}\rightarrow t_{N}|\boldsymbol{z}_{N-1}) \nonumber\\
&\hspace{3mm}+(1-\gamma)\Delta(t_{N-1}\rightarrow t_{N-2}|\boldsymbol{z}_{N-1})-\gamma\Delta(t_{N-2}\rightarrow t_{N-1}|\boldsymbol{z}_{N-2}) \nonumber\\
&\hspace{3mm}+ \Delta(t_{N-2}\rightarrow t_{N-3}|\boldsymbol{z}_{N-2}) \nonumber \\
&= \boldsymbol{z}_{N}+\sum_{j=N-1}^{N} \Big[\frac{1-(-\gamma)^{j-i}}{1+\gamma}\Delta(t_{j}\rightarrow t_{j-1}|\boldsymbol{z}_{j}) - \frac{\gamma+(-\gamma)^{j-i}}{1+\gamma}\Delta(t_{j-1}\rightarrow t_{j}|\boldsymbol{z}_{j-1}))\Big] \nonumber \\
&\hspace{3mm}+\Delta(t_{N-2}\rightarrow t_{N-3}|\boldsymbol{z}_{N-2}). \nonumber
\end{align}
where step $(a)$ makes use of (\ref{equ:z_N_difference_proof}). It is clear that the above expression for $\boldsymbol{z}_{N-3}$ is identical to (\ref{equ:BDIA_fb_int_ave}) with $i=N-3$.


The final step is to first assume that $\boldsymbol{z}_k$ and $\boldsymbol{z}_{k+1}$  can be represented by (\ref{equ:BDIA_fb_int_ave}) with $i=k$ and $i=k+1$, respectively. That is,  
\begin{align}
\boldsymbol{z}_{k}&= \boldsymbol{z}_{N}+\sum_{j=k+2}^{N} \Big[\frac{1-(-\gamma)^{j-k}}{1+\gamma}\Delta(t_{j}\rightarrow t_{j-1}|\boldsymbol{z}_{j}) - \frac{\gamma+(-\gamma)^{j-k}}{1+\gamma}\Delta(t_{j-1}\rightarrow t_{j}|\boldsymbol{z}_{j-1}))\Big] \nonumber \\
&\hspace{3mm}+\Delta(t_{k+1}\rightarrow t_{k}|\boldsymbol{z}_{k+1}) \label{equ:z_k_proof} \\
\boldsymbol{z}_{k+1}&= \boldsymbol{z}_{N}+\sum_{j=k+3}^{N} \Big[\frac{1-(-\gamma)^{j-k-1}}{1+\gamma}\Delta(t_{j}\rightarrow t_{j-1}|\boldsymbol{z}_{j}) - \frac{\gamma+(-\gamma)^{j-k-1}}{1+\gamma}\Delta(t_{j-1}\rightarrow t_{j}|\boldsymbol{z}_{j-1}))\Big] \nonumber \\
&\hspace{3mm}+\Delta(t_{k+2}\rightarrow t_{k+1}|\boldsymbol{z}_{k+2}). 
\label{equ:z_kplus1_proof}
\end{align}
We will then show in the following that $\boldsymbol{z}_{k-1}$ is again identical to (\ref{equ:BDIA_fb_int_ave}) with $i=k-1$. 
From (\ref{equ:BDIA_i_ave}), $\boldsymbol{z}_{k-1}$ can be computed in terms of $(\boldsymbol{z}_{k},\boldsymbol{z}_{k+1})$ as 
\begin{align}
\boldsymbol{z}_{k-1} &= \boldsymbol{z}_{k+1}-(1-\gamma)(\boldsymbol{z}_{k+1}-\boldsymbol{z}_k) -\gamma\Delta(t_{k}\rightarrow t_{k+1}|\boldsymbol{z}_k) + \Delta(t_{k}\rightarrow t_{k-1}|\boldsymbol{z}_k) \nonumber \\
 &= \gamma\boldsymbol{z}_{k+1}+(1-\gamma)\boldsymbol{z}_k-\gamma\Delta(t_{k}\rightarrow t_{k+1}|\boldsymbol{z}_k) + \Delta(t_{k}\rightarrow t_{k-1}|\boldsymbol{z}_k) \label{equ:z_kminus1_proof} 
\end{align}

Plugging (\ref{equ:z_kplus1_proof}) and (\ref{equ:z_k_proof}) into (\ref{equ:z_kminus1_proof}) produces 
\begin{align} 
\boldsymbol{z}_{k-1} &= \gamma\boldsymbol{z}_{N}+\gamma\sum_{j=k+3}^{N} \Big[\frac{1-(-\gamma)^{j-k-1}}{1+\gamma}\Delta(t_{j}\rightarrow t_{j-1}|\boldsymbol{z}_{j}) - \frac{\gamma+(-\gamma)^{j-k-1}}{1+\gamma}\Delta(t_{j-1}\rightarrow t_{j}|\boldsymbol{z}_{j-1}))\Big] \nonumber \\
&\hspace{3mm}+\gamma\Delta(t_{k+2}\rightarrow t_{k+1}|\boldsymbol{z}_{k+2}) \nonumber\\
&+(1-\gamma)\boldsymbol{z}_{N}+(1-\gamma)\sum_{j=k+2}^{N} \Big[\frac{1-(-\gamma)^{j-k}}{1+\gamma}\Delta(t_{j}\rightarrow t_{j-1}|\boldsymbol{z}_{j}) - \frac{\gamma+(-\gamma)^{j-k}}{1+\gamma}\Delta(t_{j-1}\rightarrow t_{j}|\boldsymbol{z}_{j-1}))\Big] \nonumber \\
&\hspace{3mm}+(1-\gamma)\Delta(t_{k+1}\rightarrow t_{k}|\boldsymbol{z}_{k+1}) \nonumber\\
&\hspace{3mm}-\gamma\Delta(t_{k}\rightarrow t_{k+1}|\boldsymbol{z}_k) + \Delta(t_{k}\rightarrow t_{k-1}|\boldsymbol{z}_k) \nonumber \\
&= \boldsymbol{z}_{N}+\sum_{j=k+3}^{N} \Big[\frac{(\gamma+(-\gamma)^{j-k})+(1-\gamma)(1-(-\gamma)^{j-k})}{1+\gamma}\Delta(t_{j}\rightarrow t_{j-1}|\boldsymbol{z}_{j}) \nonumber\\
&\hspace{25mm}- \frac{\gamma^2-(-\gamma)^{j-k}+(1-\gamma)(\gamma+(-\gamma)^{j-k})}{1+\gamma}\Delta(t_{j-1}\rightarrow t_{j}|\boldsymbol{z}_{j-1}))\Big] \nonumber \\
&+ \Big[\frac{\gamma+\gamma^2+(1-\gamma)(1-(-\gamma)^{2})}{1+\gamma}\Delta(t_{k+2}\rightarrow t_{k+1}|\boldsymbol{z}_{k+2}) - \frac{(1-\gamma)(\gamma+(-\gamma)^{2})}{1+\gamma}\Delta(t_{k+1}\rightarrow t_{k+2}|\boldsymbol{z}_{k+1}))\Big] \nonumber \\
&\hspace{3mm}+(1-\gamma)\Delta(t_{k+1}\rightarrow t_{k}|\boldsymbol{z}_{k+1}) -\gamma\Delta(t_{k}\rightarrow t_{k+1}|\boldsymbol{z}_k) \nonumber\\
&\hspace{3mm} + \Delta(t_{k}\rightarrow t_{k-1}|\boldsymbol{z}_k) \nonumber \\
&= \boldsymbol{z}_{N}+\sum_{j=k+3}^{N} \Big[\frac{1-(-\gamma)^{j-k+1}}{1+\gamma}\Delta(t_{j}\rightarrow t_{j-1}|\boldsymbol{z}_{j}) - \frac{\gamma+(-\gamma)^{j-k+1}}{1+\gamma}\Delta(t_{j-1}\rightarrow t_{j}|\boldsymbol{z}_{j-1}))\Big] \nonumber \\
&+ \Big[\frac{1-(-\gamma)^3 }{1+\gamma}\Delta(t_{k+2}\rightarrow t_{k+1}|\boldsymbol{z}_{k+2}) - \frac{\gamma+(-\gamma)^3}{1+\gamma}\Delta(t_{k+1}\rightarrow t_{k+2}|\boldsymbol{z}_{k+1}))\Big] \nonumber \\
&\hspace{3mm}+(1-\gamma)\Delta(t_{k+1}\rightarrow t_{k}|\boldsymbol{z}_{k+1}) -\gamma\Delta(t_{k}\rightarrow t_{k+1}|\boldsymbol{z}_k) \nonumber\\
&\hspace{3mm} + \Delta(t_{k}\rightarrow t_{k-1}|\boldsymbol{z}_k) \nonumber \\
&= \boldsymbol{z}_{N}+\sum_{j=k+1}^{N} \Big[\frac{1-(-\gamma)^{j-k+1}}{1+\gamma}\Delta(t_{j}\rightarrow t_{j-1}|\boldsymbol{z}_{j}) - \frac{\gamma+(-\gamma)^{j-k+1}}{1+\gamma}\Delta(t_{j-1}\rightarrow t_{j}|\boldsymbol{z}_{j-1}))\Big] \nonumber \\
&\hspace{3mm} + \Delta(t_{k}\rightarrow t_{k-1}|\boldsymbol{z}_k) \label{equ:z_kminus1_proof_2nd} 
\end{align}
We can confirm from (\ref{equ:z_kminus1_proof_2nd}) that $\boldsymbol{z}_{k-1}$ is indeed identical to (\ref{equ:BDIA_fb_int_ave}) with $i=k-1$. The proof is complete.

\section{Coupled Bidirectional Integration Approximation (CBDIA) as an Extension of EDICT }
\label{appendix:CBDIA}

In this section, we introduce the concept of bidirectional integration approximation into EDICT. The new sampling method is referred to as \emph{coupled bidirectional integration approximation (CBDIA)}. Similarly to EDICT, CBDIA also needs to  evaluate the neural network model $\hat{\boldsymbol{\epsilon}}_{\boldsymbol{\theta}}$ two times in general per timestep. We will show that CBDIA under a special parameter-setup reduces to BDIA with $\gamma=1$.  

\subsection{Update expressions of CBDIA}
In this subsection, we present the sampling update-expressions of CBDIA. Suppose at timestep $t_i$, we have a pair of diffusion states $(\boldsymbol{z}_i,\boldsymbol{y}_i)$. The pair of diffusion states $(\boldsymbol{z}_{i-1},\boldsymbol{y}_{i-1})$ at timestep $t_{i-1}$ are computed to be   
\begin{align}
\boldsymbol{w}_{i-1} &= \boldsymbol{\boldsymbol{z}_i}+\overbrace{\Delta(t_i\rightarrow t_{i-1}|\boldsymbol{y}_i)}^{\textcolor{blue}{\approx \int_{t_i}^{t_{i-1}}\boldsymbol{d}(\boldsymbol{z},t)dt}} \label{equ:CBDIA_1} \\
\boldsymbol{v}_{i-1} &=\boldsymbol{y}_i - \underbrace{\Delta(t_{i-1}\rightarrow t_i|\boldsymbol{w}_{i-1}) }_{\textcolor{blue}{\approx -\int_{t_{i}}^{t_{i-1}}\boldsymbol{d}(\boldsymbol{z},t)dt}}  \label{equ:CBDIA_2} 
\end{align}
\begin{align}
\boldsymbol{z}_{i-1} &=\gamma_1\boldsymbol{w}_{i-1} + (1-\gamma_1)\boldsymbol{v}_{i-1}\label{equ:CBDIA_3} \\
\boldsymbol{y}_{i-1} &=\gamma_2\boldsymbol{w}_{i-1} + (1-\gamma_2)\boldsymbol{v}_{i-1},\label{equ:CBDIA_4} 
\end{align}
where $\gamma_1,\gamma_2\in [0,1]$ and $\gamma_1\neq \gamma_2$. 
It is seen from the above equations that the computation of $\boldsymbol{w}_{i-1}$ involves forward integration approximation $\Delta(t_i\rightarrow t_{i-1}|\boldsymbol{y}_i)$ while $\boldsymbol{v}_{i-1}$ is computed by using the backward integration approximation $\Delta(t_{i-1}\rightarrow t_{i}|\boldsymbol{w}_{i-1})$.

Now we show that (\ref{equ:CBDIA_1})-(\ref{equ:CBDIA_4}) are also invertable. Suppose we obtain $(\boldsymbol{z}_{i-1},\boldsymbol{y}_{i-1})$ at timestep $t_{i-1}$. The pair of diffusion states  $(\boldsymbol{z}_{i},\boldsymbol{y}_{i})$ at timestep $t_i$ can be easily computed to be 
\begin{align}
\boldsymbol{v}_{i-1} &= \frac{\gamma_1 \boldsymbol{y}_{i-1}-\gamma_2\boldsymbol{z}_{i-1}}{\gamma_1-\gamma_2} \label{equ:CBDIA_inv_1} \\
\boldsymbol{w}_{i-1} &= \frac{(1-\gamma_1) \boldsymbol{y}_{i-1}-(1-\gamma_2)\boldsymbol{z}_{i-1}}{\gamma_2-\gamma_1} \label{equ:CBDIA_inv_2} \\
\boldsymbol{y}_{i} &=\boldsymbol{v}_{i-1} + \Delta(t_{i-1}\rightarrow t_i|\boldsymbol{w}_{i-1}) \label{equ:CBDIA_inv_3} \\
\boldsymbol{z}_{i} &= \boldsymbol{\boldsymbol{w}_{i-1}}-\Delta(t_i\rightarrow t_{i-1}|\boldsymbol{y}_i).\label{equ:CBDIA_inv_4} 
\end{align}

\subsection{CBDIA including BDIA with $\gamma=1$ as a special case}
In this subsection, we show that CBDIA with  $(\gamma_1,\gamma_2)=(0,1)$ reduces to BDIA with $\gamma=1$. It is straightforward that (\ref{equ:CBDIA_1})-(\ref{equ:CBDIA_4}) under $(\gamma_1,\gamma_2)=(0,1)$ take a special form of 
\begin{align}
\boldsymbol{y}_{i-1} &= \boldsymbol{\boldsymbol{z}_i}+\Delta(t_i\rightarrow t_{i-1}|\boldsymbol{y}_i) \label{equ:CBDIA_1_BDIA} \\
\boldsymbol{z}_{i-1} &=\boldsymbol{y}_i - \Delta(t_{i-1}\rightarrow t_i|\boldsymbol{y}_{i-1}),  \label{equ:CBDIA_2_BDIA} 
\end{align}
which, under the assumption of  $\boldsymbol{z}_N=\boldsymbol{y}_N$, can be reformulated  as
\begin{align}
\boldsymbol{y}_{N-1} &=  \boldsymbol{y}_{N} +  \Delta(t_N\rightarrow t_{N-1}|\boldsymbol{y}_{N}) \label{equ:CBDIA_1_BDIA_0} \\
\boldsymbol{y}_{i-1} &=  \boldsymbol{y}_{i+1} - \Delta(t_i\rightarrow t_{i+1}|\boldsymbol{y}_{i}) + \Delta(t_i\rightarrow t_{i-1}|\boldsymbol{y}_{i}) & i< N-1\label{equ:CBDIA_1_BDIA_1}  \\
\boldsymbol{z}_{i-1} &=   \boldsymbol{y}_{i} -\Delta(t_{i-1}\rightarrow t_{i}|\boldsymbol{y}_{i-1}) & \hspace{-10mm}i<N. \label{equ:CBDIA_2_BDIA_2}
\end{align}
One can easily conclude from (\ref{equ:CBDIA_1_BDIA_1}) that the update expression for $\boldsymbol{y}_{i-1}$ is identical to that of BDIA with $\gamma=1$. 
It is not difficult to derive the update expressions for  $(\boldsymbol{y}_i,\boldsymbol{z}_i)$ in terms of the historical integration approximations. We summarize the results in a lemma below: 
\begin{lemma}
Suppose $\boldsymbol{z}_N=\boldsymbol{y}_N$, $\{(\boldsymbol{y}_i,\boldsymbol{z}_i)| i<N\}$ under the update procedure of (\ref{equ:CBDIA_1_BDIA_0})-(\ref{equ:CBDIA_2_BDIA_2}) can be represented in terms of $\{\Delta(t_j\rightarrow t_{j-1}|\boldsymbol{y}_j), \Delta(t_{j-1}\rightarrow t_{j}|\boldsymbol{y}_{j-1})\}_{j=N}^{j=i-1}$ as 
\begin{align}
\boldsymbol{y}_i =& \boldsymbol{y}_N + \Delta(t_{N}-t_{N-1}|\boldsymbol{y}_{N})\,\mathrm{mod}\,(N{-}i, 2) \hspace{-0.6mm}\nonumber \\
&+\hspace{-0.6mm} \sum_{j=i+1}^{N-1}(-\Delta(t_j\hspace{-0.6mm}\rightarrow \hspace{-0.6mm}t_{j+1}|\boldsymbol{y}_j)\hspace{-0.6mm}+\hspace{-0.6mm}\Delta(t_j\hspace{-0.6mm}\rightarrow \hspace{-0.6mm} t_{j\hspace{-0.6mm}-\hspace{-0.6mm}1}|\boldsymbol{y}_j))\,\mathrm{mod}\,(j\hspace{-0.6mm}-\hspace{-0.6mm}i,2)  \label{equ:CBDIA_1_BDIA_3} \\
\boldsymbol{z}_{i} =&   \boldsymbol{y}_N + \Delta(t_{N}-t_{N-1}|\boldsymbol{y}_{N})\,\mathrm{mod}\,(N{-}i-1, 2) \hspace{-0.6mm}\nonumber\\
&+\hspace{-0.6mm} \sum_{j=i+2}^{N-1}(-\Delta(t_j\hspace{-0.6mm}\rightarrow \hspace{-0.6mm}t_{j+1}|\boldsymbol{y}_j)\hspace{-0.6mm}+\hspace{-0.6mm}\Delta(t_j\hspace{-0.6mm}\rightarrow \hspace{-0.6mm} t_{j\hspace{-0.6mm}-\hspace{-0.6mm}1}|\boldsymbol{y}_j))\,\mathrm{mod}\,(j\hspace{-0.6mm}-\hspace{-0.6mm}i-1,2) -\Delta(t_{i}\rightarrow t_{i+1}|\boldsymbol{y}_{i}), \label{equ:CBDIA_1_BDIA_4}
\end{align}
where the directions of integration approximations for $\boldsymbol{z}_i$ are opposite of those for $\boldsymbol{y}_i$. 
\label{lemma:CBDIA_BDIA}
\end{lemma}
The results in Lemma~\ref{lemma:CBDIA_BDIA} show that CBDIA with $(\gamma_1,\gamma_2)=(0,1)$ indeed reduces to BDIA with $\gamma=1$.  

\subsection{CBDIA with $(\gamma_1,\gamma_2)=(1,0)$}
In this subsection, we consider the special case of CBDIA with $(\gamma_1,\gamma_2)=(1,0)$. It is immediate that (\ref{equ:CBDIA_1})-(\ref{equ:CBDIA_4}) under the setup $(\gamma_1,\gamma_2)=(1,0)$ can be simplified to be
\begin{align}
\boldsymbol{z}_{i-1} &= \boldsymbol{\boldsymbol{z}_i}+\Delta(t_i\rightarrow t_{i-1}|\boldsymbol{y}_i)\label{equ:CBDIA_1_back} \\
\boldsymbol{y}_{i-1} &=\boldsymbol{y}_i - \Delta(t_{i-1}\rightarrow t_i|\boldsymbol{z}_{i-1}). \label{equ:CBDIA_2_back} 
\end{align}
One can easily represent $(\boldsymbol{y}_i,\boldsymbol{z}_i)$ in terms of the historical integration approximations.  We summarize the results in a lemma below: 
\begin{lemma}
Suppose $\boldsymbol{z}_N=\boldsymbol{y}_N$, $\{(\boldsymbol{y}_i,\boldsymbol{z}_i)| i<N\}$ under the update procedure of (\ref{equ:CBDIA_1_back})-(\ref{equ:CBDIA_2_back}) can be represented in terms of $\{\Delta(t_j\rightarrow t_{j+1}|\boldsymbol{z}_j)\}_{j=N-1}^{i}$ and $\{\Delta(t_j\rightarrow t_{j-1}|\boldsymbol{y}_j)\}_{j=N}^{i+1}$ as 
\begin{align}
\boldsymbol{y}_i & = \boldsymbol{y}_N - \sum_{j=i}^{N-1}\Delta(t_j\rightarrow t_{j+1}|\boldsymbol{z}_j) \nonumber\\
\boldsymbol{z}_i & = \boldsymbol{z}_N + \sum_{j=i+1}^{N}\Delta(t_j\rightarrow t_{j-1}|\boldsymbol{y}_j). \nonumber
\end{align}
\end{lemma}

If we let $(\gamma_1,\gamma_2)=(\phi, 1-\phi)$ in CBDIA, the update expressions for $(\boldsymbol{y}_i,\boldsymbol{z}_i)$ should include an average of forward and backward integration approximations per historical timeslot $[t_{j+1}, t_{j}]$, $j\leq i$. 

\subsection{CBDIA by using DDIM update procedure (CBDIA-DDIM)}
In this subsection, we consider the special case that the DDIM update procedure is used to approximate $\Delta(t_i\rightarrow t_{i-1}|\boldsymbol{y}_i)$ and $\Delta(t_{i-1}\rightarrow t_{i}|\boldsymbol{w}_{i-1})$ in (\ref{equ:CBDIA_1})-(\ref{equ:CBDIA_2}). It is immediate that (\ref{equ:CBDIA_1})-(\ref{equ:CBDIA_4}) can be represented as
\begin{align}
\boldsymbol{w}_{i-1} &= \boldsymbol{z}_i+\overbrace{\alpha_{i-1} \left(\frac{\boldsymbol{y}_i \hspace{-0.3mm}-\hspace{-0.3mm} \sigma_{i}\hat{\boldsymbol{\epsilon}}_{\boldsymbol{\theta}}(\boldsymbol{y}_i, i) }{\alpha_{i}}\right)+\hspace{0.5mm}\sigma_{i-1}\hat{\boldsymbol{\epsilon}}_{\boldsymbol{\theta}}(\boldsymbol{y}_i, i)}^{\textcolor{blue}{\textrm{forward DDIM update}}}-\boldsymbol{y}_i \label{equ:CBDIA_DDIM_1} \\
\boldsymbol{v}_{i-1}&=\boldsymbol{y}_i - \left(\overbrace{\alpha_{i} \left(\frac{\boldsymbol{w}_{i-1} \hspace{-0.3mm}-\hspace{-0.3mm} \sigma_{i-1}\hat{\boldsymbol{\epsilon}}_{\boldsymbol{\theta}}(\boldsymbol{w}_{i-1}, i-1) }{\alpha_{i-1}}\right)+\hspace{0.5mm}\sigma_{i}\hat{\boldsymbol{\epsilon}}_{\boldsymbol{\theta}}(\boldsymbol{w}_{i-1}, i-1)}^{\textcolor{blue}{\textrm{backward DDIM update}}}-\boldsymbol{w}_{i-1}\right)  \label{equ:CBDIA_DDIM_2} \\
\boldsymbol{z}_{i-1} &= \gamma_1 \boldsymbol{w}_{i-1}+(1-\gamma_1)\boldsymbol{v}_{i-1}. \label{equ:CBDIA_DDIM_3} \\
\boldsymbol{y}_{i-1} &= \gamma_2 \boldsymbol{w}_{i-1}+(1-\gamma_2)\boldsymbol{v}_{i-1}. \label{equ:CBDIA_DDIM_4} 
\end{align}
We refer to the above sampling procedure as CBDIA-DDIM. With (\ref{equ:CBDIA_DDIM_1})-(\ref{equ:CBDIA_DDIM_4}), one can then easily derive the inverse update expressions.

\section{Design of BDIA-DPM-Solver++}
\label{appendix:BDIA-DPM-Solver}

The method DPM-Solver++\footnote{Lu et al., DPM-Solver++: Fast Solver for Guided Sampling of Diffusion Probabilistic Models, arXiv:2211.01095.} in StableDiffusion V2 has a configurable implementation with several options.
The results in Table~\ref{tab:t2i_FID_compare} were obtained by using the multi-step 2nd-order DPM-Solver++, which is the default setup in StableDiffusion. At each timestep $t_j$, the pre-trained DNN model produces an estimator $\hat{\boldsymbol{x}}_{\boldsymbol{\theta}}(\boldsymbol{z}_j, \phi=P, t_j)$ of the clean-image, where $P$ denotes the text prompt. In general, when $i<N$, the diffusion state $\boldsymbol{z}_{i-1}$ is computed by making use of the two most recent clean-image estimators $\{\hat{\boldsymbol{x}}_{\boldsymbol{\theta}}(\boldsymbol{z}_j, \phi=P, t_j)|j=i+1,i\}$ as well as the current state $\boldsymbol{z}_i$.  For simplicity, let us denote the update expression of DPM-Solver++ for computing $\boldsymbol{z}_{i-1}$ at timestep $t_i$ as 
\begin{align}
\boldsymbol{z}_{i-1} = \Gamma_{t_i\rightarrow t_{i-1}}\left(\boldsymbol{z}_i,\{\hat{\boldsymbol{x}}_{\boldsymbol{\theta}}(\boldsymbol{z}_j, \phi=P, t_j)|j=i+1, i\}\right). \label{equ:DPM_solver_i}
\end{align}

Next we present the update expressions of BDIA-DPM-Solver++, which is designed primarily for improving sampling quality rather than enabling diffusion inversion. The diffusion state $\boldsymbol{z}_{i-1}$ at timestep $t_i$ is computed to be 
\begin{align}
\boldsymbol{z}_{i-1} =& \boldsymbol{z}_{i+1}+ \overbrace{(1-\gamma)(\boldsymbol{z}_i-\boldsymbol{z}_{i+1})-\gamma \Delta(t_i\rightarrow t_{i+1}|\boldsymbol{z}_i)}^{\textcolor{blue}{\approx \int_{t_{i+1}}^{t_{i}}\boldsymbol{d}(\boldsymbol{z},t)dt}} \nonumber \\
&+(\Gamma_{t_i\rightarrow t_{i-1}}\left(\boldsymbol{z}_i,\{\hat{\boldsymbol{x}}_{\boldsymbol{\theta}}(\boldsymbol{z}_j, \phi=P, t_j)|j=i+1, i\}\right)-\boldsymbol{z}_i),
\end{align}
where $\Delta(t_i\rightarrow t_{i+1}|\boldsymbol{z}_i)$ is the backward integration approximation over the timeslot $[t_{i+1},t_i]$. In practice, one can employ DDIM to realize $\Delta(t_i\rightarrow t_{i+1}|\boldsymbol{z}_i)$. 

\begin{remark}
In principle, our BDIA technique can also be incorporated into other high-order ODE solvers such as PLMS and DEIS by following a similar design procedure for BDIA-DPM-Solver++ presented above. We omit the details here. 
\end{remark}

\vspace{-2mm}
\section{Tested Pre-trained Models for BDIA-DDIM and BDIA-EDM}
\label{appendix:pre_trained_models}
\vspace{-2mm}

 \begin{table}[h!]
\caption{\small Tested pre-trained models in Table~\ref{tab:t2i_FID_compare}, \ref{tab:BDIA_DDIM} and \ref{tab:BDIA_EDM}  \vspace{0mm} } 
\label{tab:pre_trained_models}
\centering
\begin{tabular}{|c|c|}
\hline
 \small{$\begin{array}{c}\textrm{BDIA-DDIM}\\ \textrm{(text-to-image sampling)}\end{array}$ }  &  {\scriptsize $\begin{array}{c}\textrm{v2-1}\_\textrm{512-ema-pruned.ckpt} \\ 
\textrm{ from \url{ https://huggingface.co/stabilityai/stable-diffusion-2-1-base/tree/main}}\end{array}$} \\
\hline
 \small{$\begin{array}{c}\textrm{BDIA-DDIM}\\ \textrm{(unconditional sampling)}\end{array}$ }  & \small {$\begin{array}{c}\textrm{ddim}\_\textrm{cifar10.ckpt} 
\textrm{ from \url{https://github.com/luping-liu/PNDM}} \\
\textrm{ddim}\_\textrm{celeba.ckpt} 
\textrm{ from \url{https://github.com/luping-liu/PNDM}}\end{array}$}  \\ 
\hline 
 \small{\textrm{BDIA-EDM}}  & \small { $\begin{array}{c}\textrm{ 
 edm-cifar10-32x32-cond-vp.pkl from \url{https://github.com/NVlabs/edm}} \\
 \textrm{ 
 edm-afhqv2-64x64-uncond-vp.pkl from \url{https://github.com/NVlabs/edm}} \\
 \textrm{ 
 edm-ffhq-64x64-uncond-vp.pkl from \url{https://github.com/NVlabs/edm}} \\
 \textrm{ 
 edm-imagenet-64x64-cond-adm.pkl from \url{https://github.com/NVlabs/edm}}
 \end{array}$ }  \\
\hline 
\end{tabular}
\end{table}

\end{document}